\documentclass{article}

\PassOptionsToPackage{numbers, compress}{natbib}

\usepackage[final]{staix_2026}         

\usepackage[utf8]{inputenc}
\usepackage[T1]{fontenc}
\usepackage[colorlinks,citecolor=blue,urlcolor=blue]{hyperref}
\usepackage{url}
\usepackage{booktabs}
\usepackage{amsfonts}
\usepackage{nicefrac}
\usepackage{microtype}
\usepackage{xcolor}
\usepackage{graphicx}
\usepackage{amsmath,amssymb,amsthm}
\usepackage{algorithm}
\usepackage{algorithmic}
\usepackage{soul}
\usepackage{enumitem}
\usepackage{bm}
\usepackage{subcaption}
\usepackage{mathtools}
\usepackage{booktabs}
\usepackage{todonotes}
\usepackage{aliascnt}
\usepackage{multirow}
\usepackage{tcolorbox}
\usepackage{wrapfig}
\usepackage[capitalize,noabbrev]{cleveref}

\usepackage{amsmath,amsfonts,bm}

\def\eqref#1{equation~\ref{#1}}

\def\1{\bm{1}}

\DeclareMathAlphabet{\mathsfit}{\encodingdefault}{\sfdefault}{m}{sl}
\SetMathAlphabet{\mathsfit}{bold}{\encodingdefault}{\sfdefault}{bx}{n}

\newcommand{\R}{\mathbb{R}}

\newcommand{\KL}{D_{\mathrm{KL}}}

\DeclareMathOperator*{\argmax}{arg\,max}
\DeclareMathOperator*{\argmin}{arg\,min}

\theoremstyle{plain}
\newtheorem{theorem}{Theorem}

\newaliascnt{lemma}{theorem}
\newtheorem{lemma}[lemma]{Lemma}
\aliascntresetthe{lemma}

\newaliascnt{proposition}{theorem}
\newtheorem{proposition}[proposition]{Proposition}
\aliascntresetthe{proposition}

\newaliascnt{corollary}{theorem}

\newtheorem{assumption}{Assumption}
\aliascntresetthe{corollary}

\crefname{theorem}{Theorem}{Theorems}
\crefname{lemma}{Lemma}{Lemmas}
\crefname{proposition}{Proposition}{Propositions}
\crefname{corollary}{Corollary}{Corollaries}
\crefname{definition}{Definition}{Definitions}
\crefname{remark}{Remark}{Remarks}
\crefname{assumption}{Assumption}{Assumptions}

\newcommand{\bmu}{\bm{\mu}}

\newcommand{\PP}{\mathbb{P}}

\newcommand{\Mc}{\mathcal{M}}

\newcommand{\Oc}{\mathcal{O}}
\newcommand{\Sc}{\mathcal{S}}

\floatname{algorithm}{Algorithm}
\renewcommand{\algorithmicrequire}{\textbf{Input:}}
\renewcommand{\algorithmicensure}{\textbf{Output:}}

\makeatletter
\renewcommand{\eqref}[1]{\textup{\tagform@{\ref{#1}}}}
\makeatother

\title{Segmenting Human--LLM Co-authored Text via Change Point Detection}

\author{%
Mengchu~Li$^*$ \\
School of Mathematics \\
University of Birmingham \\
Birmingham, UK \\
\texttt{m.li.15@bham.ac.uk}
\And
Jin~Zhu\thanks{Mengchu~Li and Jin~Zhu contributed equally to this paper and are listed in the alphabetical order.} \\
School of Mathematics \\
University of Birmingham \\
Birmingham, UK \\
\texttt{j.zhu.7@bham.ac.uk}
\And
Jinglai~Li \\
School of Mathematics \\
University of Birmingham \\
Birmingham, UK \\
\texttt{j.li.10@bham.ac.uk}
\And
Chengchun Shi\thanks{Corresponding authors} \\
Department of Statistics \\
London School of Economics and Political Science  \\
London, UK \\
\texttt{c.shi7@lse.ac.uk}
}

\begin{document}

\maketitle

\begin{abstract}
The rise of large language models (LLMs) has created an urgent need to distinguish between human-written and LLM-generated text to ensure authenticity and societal trust. Existing detectors typically provide a binary classification for an entire passage; however, this is insufficient for human--LLM co-authored text, where the objective is to localize specific segments authored by humans or LLMs. To bridge this gap, we propose algorithms to segment text into human- and LLM-authored pieces. Our key observation is that such a segmentation task is conceptually similar to classical change point detection in time-series analysis. Leveraging this analogy, we adapt change point detection to LLM-generated text detection, develop a weighted algorithm and a generalized algorithm to accommodate heterogeneous detection score variability, and establish the minimax optimality of our procedure. Empirically, we demonstrate the strong performance of our approach against a wide range of existing baselines. The python implementation of our proposal is available at \url{https://github.com/Mamba413/DetectLLMSegmentation}.
\end{abstract}


\section{Introduction}
\vspace{-7pt}
State-of-the-art large language models (LLMs), including GPT-5 \citep{singh2025openai}, Gemini \citep{comanici2025gemini}, and Grok-4 \citep{grok2025}, exhibit strong capabilities in following human instructions, performing complex reasoning, and generating text at scale. These powerful tools have been deeply integrated into both professional and daily workflows, rendering LLM-generated content pervasive across various domains, such as academic literature, technical reports, and student assignments. This growth has spurred a line of research on LLM-generated text detection (see Appendix \ref{sec:app-literature} for a review). 

These works study detection as a binary classification problem, determining whether an entire passage is human- or LLM-authored. However, in practice, LLM-generated text is rarely used without modification. Users often revise, edit, or blend it with their own writing. For instance, in technical reports, users may draft the core methodological sections themselves while relying on LLMs to generate descriptive text for tables and figures. In creative writing, authors may collaboratively generate stories with LLMs \citep{xie2023next}. Such human--LLM co-authoring settings have become increasingly common as LLMs continue to evolve.

In these ``hybrid'' settings, a binary classification of an entire passage as human- or LLM-authored is insufficient. A more meaningful question is whether we can segment the document into human- and LLM-authored pieces to identify which portions are generated by LLMs. However, as discussed below, only a few works have considered LLM-generated text localization.

This paper advances the state of LLM text detection by making the following \textbf{contributions}:\vspace{-0.5em}
\begin{itemize}[leftmargin=*]
\item We identify that the problem of segmenting human--LLM co-authored text is conceptually closely connected to \ul{\textit{change point detection}} problems in time-series analysis. This perspective motivates a new framework that adapts change point detection for LLM-generated text localization. 
\item We propose three change point detection algorithms tailored to human--LLM text segmentation. We begin with a \ul{\textit{vanilla}} adaptation of classical change point detection and reveal its limitations in settings with heterogeneous detection score variability. To overcome this limitation, we develop a \ul{\textit{weighted}} algorithm and a \ul{\textit{generalized}} algorithm that encodes sentence-specific information and substantially improve segmentation accuracy.
\item We derive \ul{\textit{estimation error bounds}} for both vanilla and weighted algorithms (Theorems~\ref{thm:cusum} and~\ref{thm:cusum-hetero}), under the assumption that each contiguous segment is  either human-written or LLM-generated. These results show that the latter achieves more accurate estimation under less restrictive signal-to-noise ratio conditions. Together with a \ul{\textit{minimax lower bound}} that characterizes the intrinsic difficulty of the segmentation problem (Theorem~\ref{lemma:lowerbound}), we demonstrate that the weighted algorithm achieves this limit, establishing its \ul{\textit{minimax optimality}}. We further extend such optimality guarantees to the generalized algorithm under appropriate conditions. 
\item We evaluate the empirical performance of the proposed algorithms across a broad range of experimental settings (Section~\ref{sec:exp} and Appendix~\ref{sec:addition-numerical}). Results demonstrate that our proposal consistently outperforms existing baseline algorithms.
\end{itemize} 

\textbf{Literature review}. Our proposal is most closely related to the literature on LLM-generated text localization. In practical settings, users rarely adopt LLM-generated text directly. Instead, they revise or edit the generated content \citep{zhang2024llm}. This collaborative writing process is commonly referred to as \ul{\emph{human--AI co-authoring}}. LLM-generated text localization aims to localize the portions of a document that originate from an LLM. \citet{dugan2023bounary} empirically evaluate humans' ability to detect the boundary between human- and LLM-authored text. A line of recent research applies existing detectors at the sentence level to localize LLM-authored sentences \citep{wang2023seqxgpt, zeng2024aisentence, jiang2025sendetex}. In contrast, \citet{su2025haco} adopt a bottom-up approach, performing token-level detection and aggregating these classifications to determine whether a sentence is authored by a human or an LLM. In addition to these works, \citet{zhang2024machine} demonstrate that pooling information across several consecutive sentences improves performance over individual sentence-level classification. \citet{li2024segmenting} and \citet{bonnerjee2025fast} combine change point detection algorithms with watermarking techniques to obtain text segmentation. 

\section{Preliminaries}\label{sec:preliminary}
\vspace{-7pt}
\textbf{Detection statistics for LLM-generated text}. 
Our procedure relies on certain detection statistics to distinguish between LLM- and human-authored text, which we first introduce in this section. While allowing for general detection statistics, we focus, for illustration, on the zero-shot FastDetectGPT statistic~\citep{bao2024fastdetectgpt}, which leverages next-token prediction probabilities from a pretrained language model. Given a text segment $\bm{Z}$ containing $n$ tokens, let \( \bm{Z}_{< t} \coloneqq (Z_1, Z_2, \ldots, Z_{t-1}) \). Consider the following detection statistic:
\begin{align*}
\phi(\bm{Z})=\frac{1}{n}\sum_{t=1}^n \bigg\{\log \texttt{score}(Z_t | \bm{Z}_{<t})
- \mathbb{E}_{Z'_t \sim \texttt{sample}(\bullet | \bm{Z}_{<t})}
\big[\log \texttt{score}(Z'_t | \bm{Z}_{<t})\big]\bigg\},
\end{align*}
where \(\texttt{score}(Z_t | \bm{Z}_{<t})\) denotes the scoring model's logit for token $Z_t$ given $\bm{Z}_{<t}$ and
\(\texttt{sample}(\bullet | \bm{Z}_{<t})\) denotes the sampling model, which is used to sample $Z'_t$ given $\bm{Z}_{<t}$, and which may be different from the scoring model. The intuition behind this statistic is that LLM-generated text tends to yield higher values compared to human-written text \citep[Figure 1]{bao2024fastdetectgpt}. We also note that the variance of $\phi(\bm{Z})$ when $\bm{Z}$ is generated from the sampling model can be approximated by
$\hat{\sigma}^2(\bm{Z}) = n^{-2}\sum_t \textrm{Var}_{Z'_t \sim \texttt{sample}(\bullet | \bm{Z}_{<t})}
\big[\log \texttt{score}(Z'_t | \bm{Z}_{<t})\big]$.
It is clear from the formulation that as the length of the text $n$ increases, the variability of $\phi(\bm{Z})$ tends to decrease and therefore becomes more reliable. 
Further developments based on FastDetectGPT utilize training data to improve performance \citep[e.g.][]{zhou2025adadetectgpt,zhou2026detecting}, and we consider those ML-based variants in our experiments as well.

\textbf{CUSUM and NOT for change point detection}. We next review the CUSUM statistic and the narrowest-over-thresholding \citep[NOT,][]{baranowski2019narrowest} algorithm for classical change point detection, upon which our proposal is built. 
Let $\bm{Y} = (Y_1, \dotsc, Y_N)$ denote a time series of $N$ observations indexed sequentially. The goal of change point detection is to identify structural changes within this sequence. For example, letting $\mu_i = \mathbb{E}(Y_i)$ represent the expected value of each observation, we seek to determine the locations $t$ where $\mu_t \neq \mu_{t+1}$. Throughout the rest of the paper, we let $[N]$ denote the set of integers $\{1,\dotsc,N\}$ for $N\in \mathbb{Z}_{+}$. 

The CUSUM statistic plays a crucial role in these algorithms. For any time points $\{s,t,e\} \subset [N]$ such that $s \leq t< e$, the CUSUM statistic at $t$ over the interval $[s, e]$ is defined as
\begin{equation}\label{eq:CUSUM}
C_{s,e}^{\bm{Y}}(t)  = \sqrt{\frac{S_{s:t}S_{(t+1):e}}{S_{s:e}}} |\overline{Y}_{s:t}-\overline{Y}_{(t+1):e}|,
\end{equation}
where $\overline{Y}_{t_1:t_2}$ denotes the sample average of observations $\{Y_t\}_{t}$ in the interval $[t_1,t_2]$ and $S_{t_1:t_2} = t_2-t_1+1$ denotes the number of samples within $[t_1,t_2]$. 

It can be shown that $\max_{s\leq b<e}|C_{s,e}^{\bm{Y}}(b)|$ is the generalized likelihood ratio statistic for testing whether there is a change in the sequence $\{\mu_t\}_t$ under Gaussian assumptions \citep[e.g.,][]{baranowski2019narrowest, wang2020univariate}. Beyond the Gaussian case, the applicability of this maximal-type statistic extends to various data distributions as well. It has served as a cornerstone for a wide range of change point detection algorithms with both strong theoretical guarantees and promising empirical performance \citep[e.g.,][]{liu2021minimax,padilla2021optimal,wang2021optimal,li2022network}. 

Of particular interest is the NOT algorithm \citep{baranowski2019narrowest}, which we adapt for our task of LLM text localization. In \Cref{alg:not_algorithm}, we detail a meta-algorithm that slightly differs from the original NOT algorithm. This version allows two arbitrary statistics, $A^{\bm{Y}}_{s,e}(b)$ and $B(s,e)$, as inputs, which can be tailored to our needs. Note that with the default choice $A^{\bm{Y}}_{s,e}(b) = C^{\bm{Y}}_{s,e}(b), \,B(s,e) = e-s$, \Cref{alg:not_algorithm} reduces to the original algorithm.

\begin{algorithm}[t]
\caption{NOT-meta algorithm}\label{alg:not_algorithm}
\begin{algorithmic}[1]
\REQUIRE (i) Data vector $\bm{Y}=(Y_{1},\ldots, Y_{N})$, (ii) Threshold parameter $r$; (iii) Number of random intervals $M$; (iv) Statistics $A^{\bm{Y}}_{s,e}(b), B(s,e)$
\ENSURE Set of estimated change points $\Sc \subset\{1,\ldots,N\}$. 
\STATE Let $e \leftarrow N$, $s \leftarrow 1$ and $\Sc \leftarrow \emptyset$, apply \textbf{NOT-meta}$(s, e, r)$ described in Steps 2--15. 
\IF{$e-s < 1$} 
\STATE STOP
\ELSE
\STATE Uniformly randomly draw $M$ intervals within $[s, e]$ and form an interval set $\Mc$ 
\STATE $\Oc \leftarrow \left\{m\in\Mc: \max\limits_{s_m \leq b < e_m} A_{s_m,e_m}^{\bm{Y}}(b)  >  r\right\}$
\IF{$\Oc=\emptyset$} 
\STATE STOP
\ELSE
\STATE $m^{*} \leftarrow \argmin\limits_{m\in\Oc}B(s_m,e_m)$, $b^{*} \leftarrow \argmax\limits_{s_{m^*} \le b < e_{m^*}} A_{s_{m^*},e_{m^*}}^{\bm{Y}}(b)$, $\Sc \leftarrow \Sc\cup\{b^{*}\}$
\STATE \textbf{NOT-meta}($s,b^{*},r$)
\STATE \textbf{NOT-meta}($b^{*}+1,e,r$)
\ENDIF
\ENDIF
\end{algorithmic}
\end{algorithm}

\textbf{Problem setup}. Finally, we detail our problem setup. Consider a paragraph or an article $\bm{X}$ partitioned into $N$ parts $(X_1,\dotsc,X_N)$, each authored by either a human or an LLM. Each $X_t$ can be defined at different levels of granularity, such as paragraphs, sentences or tokens, and our framework applies to all these settings. Our goal is to segment $\bm{X}$ into human-written and LLM-generated pieces. A na{\"i}ve approach is to apply a score function $\phi$, either a zero-shot statistical measure or an ML classifier, to assess whether each $X_t$ is generated by an LLM. For instance, $X_t$ is classified as LLM-authored if $\phi(X_t)$ exceeds a certain threshold. With these classification labels at hand, we directly obtain a segmentation of the entire paragraph.

However, such an approach suffers from two limitations: (i) The optimal classification threshold is difficult to determine a priori, and the performance of the resulting algorithm can be highly sensitive to this choice. (ii) As discussed in the literature review, pooling information across consecutive $X_i$'s improves detection accuracy. However, the na{\"i}ve approach evaluates each $X_i$ individually, leading to suboptimal segmentation.

To address these limitations, we frame LLM text localization as a change point detection problem. We define change points as the indices where authorship transitions between human and LLM. Formally, let $0 = \tau_0 < \tau_1 < \dots < \tau_K < \tau_{K+1} = N$ denote $K$ unknown change points.  
These points partition the paragraph into $K+1$ segments, such that each segment
\begin{align*}
X_{(\tau_i+1):\tau_{i+1}} = (X_{\tau_i+1}, X_{\tau_i+2}, \dots, X_{\tau_{i+1}})
\end{align*}
is authored entirely by either a human or an LLM for every $i \in \{0\}\cup[K]$. 
Under this formulation, localizing human- or LLM-authored sentences is equivalent to identifying the set of change points $\{\tau_k\}_{k=1}^K$. This enables us to leverage classical change point detection for LLM text localization, as detailed in the next section. 

\section{Methodology: From change point detection to LLM text localization}\label{sec:method}
\vspace{-7pt}
This section details our proposal for adapting change point detection to LLM-generated text localization. We begin by introducing a vanilla adaptation of the NOT algorithm (see Algorithm~\ref{alg:not_algorithm}) and discussing its limitations. To address these limitations, we next propose a weighted algorithm and a generalized algorithm (deferred to Appendix~\ref{sec:textCP}). Finally, we detail the procedure for localizing human--LLM text based on the resulting change point estimates.

\textbf{Vanilla change point detection}. 
Having introduced the CUSUM statistic and the NOT algorithm that builds upon it, a natural approach to segmenting human--LLM co-authored text is to transform each $X_t$ in $\bm{X}$ into a scalar $\phi(X_t)$ using an existing detection score function $\phi$, such as the FastDetectGPT statistic. This transformation yields a one-dimensional time series $(\phi(X_1),\dotsc, \phi(X_N))$, enabling the application of standard change point detection, such as the default NOT algorithm discussed in \Cref{sec:preliminary}, to identify all change points $\tau_i$. We refer to this approach as \textbf{VCP}, short for \underline{V}anilla adaptation of \underline{C}hange \underline{P}oint detection.

\begin{wrapfigure}{r}{0.5\linewidth}
\centering
\vspace*{-12pt}
\includegraphics[width=1.0\linewidth]{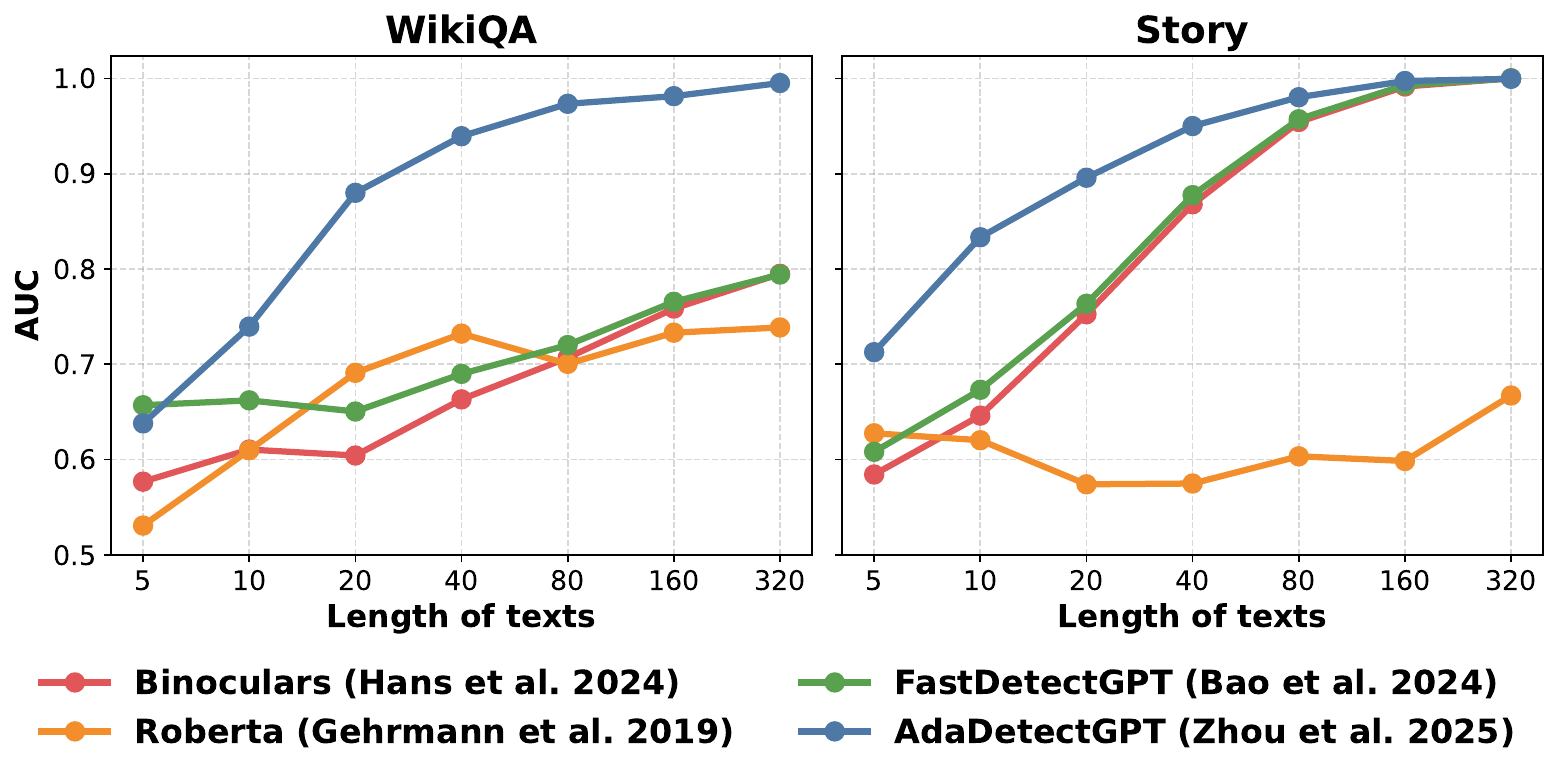}
\vspace*{-15pt}
\caption{\small AUCs of various detectors on the WikiQA and Story datasets with varying lengths of input texts. RoBERTa and AdaDetectGPT are two ML-based LLM detectors, while Binoculars and FastDetectGPT are two zero-shot detectors.}
\label{fig:length-vs-auc}
\vspace*{-1em}
\end{wrapfigure}
However, directly applying the default NOT algorithm or other off-the-shelf change point detection algorithms is often suboptimal. In sentence-level detection where each \(X_t\) is a sentence, VCP fails to account for the heterogeneous signal strength in $\phi(X_i)$ often resulting from varying sentence lengths. 
Specifically, both theoretical results \citep{zhou2025adadetectgpt} and empirical findings \citep{bao2024fastdetectgpt} suggest that detection accuracy increases with the length of the text. This relationship is further illustrated in Figure~\ref{fig:length-vs-auc}, which reports the area under the curve (AUC) for both ML-based LLM detectors and zero-shot LLM detectors across varying text lengths. The results show that detection accuracy generally increases with the input length. Consequently, change points are easier to detect when they occur near longer sentences and harder to detect near shorter sentences. 

Even if each $X_i$ has the same number of tokens, the variability of the detection score $\phi(X_i)$ can still be heterogeneous, making some segments $X_i$ easier to detect than others. Despite this, VCP treats the level of difficulty equally across the time series, which leads to inefficient change point detection. Mathematically, this limitation stems from the use of standard CUSUM statistics in \eqref{eq:CUSUM}, which treats each observation uniformly.

Our proposed methodology, detailed in the following section, explicitly scales the contribution of each score $\phi(X_i)$ in the CUSUM statistic according to its variability, measured by either a built-in estimate or sentence length. This scaling leads to more efficient and, in theory, \ul{\textit{minimax optimal}} segmentation for human--LLM co-authored text (see \Cref{sec:theory}).

\textbf{Weighted change point detection}. 
In this section, we introduce \textbf{WCP}, a \underline{W}eighted \underline{C}hange \underline{P}oint detection algorithm, to overcome the limitation of VCP. As discussed earlier, \ul{\textit{not all sentences and segments are created equal}} for authorship identification. Rather, the difficulty of detection is inherently tied to the variability of their detection score. WCP exploits this observation by assigning larger weights to more informative sentence-level detection scores and prioritizing these observations for more accurate change point detection.

Specifically, for a generic time series $\bm{Y} = (Y_1,\dotsc,Y_N)$ and a weight vector $w = (w_1,\dotsc,w_N)$, the algorithm relies on the following weighted CUSUM statistic, defined as:
\begin{equation}\label{eq:weighted_CUSUM}
W_{s,e}^{\bm{Y}}(t) = \sqrt{\frac{S_{s:t}^w \,S_{(t+1):e}^w}{S_{s:e}^w}}|\overline{Y}_{s:t}^w
- \overline{Y}_{(t+1):e}^w|,
\end{equation}
where $S_{t_1:t_2}^w = \sum_{i=t_1}^{t_2} w_i$ represents the cumulative weight, and $ \overline{Y}_{t_1:t_2}^{w} = (S_{t_1:t_2}^w)^{-1}\sum_{i=t_1}^{t_2}w_iY_i$
denotes the weighted average of the observations over the interval $[t_1, t_2]$. By construction, a relatively larger weight $w_i$ amplifies the impact of observation $Y_i$ on the statistic. When $w_i=1$ for all $i \in [N]$, the weighted CUSUM is reduced to the standard CUSUM statistic in \eqref{eq:CUSUM}. 

WCP applies the meta-algorithm (\Cref{alg:not_algorithm}) with $A_{s,e}^{\bm{Y}}(b) = W_{s,e}^{\bm{Y}}(b)$ and $B(s,e) = S^w_{s:e}$ to the score sequence $(\phi(X_1),\cdots,\phi(X_N))$ for change point detection. The choice of the weight vector ${w}$ is critical, as it encodes instance-specific information from $\phi(X_i)$ into the segmentation. In our theoretical analysis, we show that setting $w_i$ to the inverse variance of each score $\phi(X_i)$ outperforms the vanilla algorithm and achieves minimax optimal segmentation. In our implementation, we may use the variance estimator described in Section \ref{sec:preliminary} as the weights and set $w_i = \{\hat{\sigma}(X_i)\}^{-2}$. For sentence-level detection, we may alternatively use sentence length as a proxy by setting $w_i = n_i^{\kappa}$, where $n_i$ denotes the number of tokens in sentence $X_i$ and $\kappa > 0$ is a hyperparameter (typically set to $1$ or $2$).

\textbf{Human-LLM text localization}. After obtaining the estimated change points $\widehat{\tau}_1 < \dotsc < \widehat{\tau}_{\widehat{K}}$, 
we could divide the whole passage into $\widehat{K}+1$ segments 
\[
X_{1:\widehat{\tau}_1},\; X_{\widehat{\tau}_1+1:\widehat{\tau}_2},\; \ldots,\; X_{(\widehat{\tau}_{\widehat{K}-1}+1): \widehat{\tau}_{\widehat{K}}},\; X_{(\widehat{\tau}_{\widehat{K}}+1):N},
\]
where $\widehat{K}$ denotes our estimated number of change points. We next apply $\phi$ to these segments to calculate their detection scores $\{\phi(X_{1:\widehat{\tau}_1}),\phi(X_{\widehat{\tau}_1+1:\widehat{\tau}_2}),\cdots,\phi(X_{(\widehat{\tau}_{\widehat{K}}+1):N})\}$, and employ a clustering algorithm (e.g., $k$-means with $k=2$) to group these segment-wise scores into two clusters. Given our choice of $\phi$ (e.g.\ the FastDetectGPT statistic) where higher values indicate a larger probability of LLM generation, the cluster with the larger mean is classified as LLM-authored, while the other cluster is identified as human-authored. If no change point is detected, we directly apply $\phi$ to the whole document to classify the entire passage.

\section{Theory}\label{sec:theory}
\vspace{-7pt}
In this section, we develop theoretical guarantees for the proposed algorithms in \Cref{sec:method}.
We begin with a summary of our main findings: \vspace{-0.5em}
\begin{itemize}[leftmargin=*]
\item \textit{\textbf{Non-asymptotic error bound}}: Theorems~\ref{thm:cusum} and~\ref{thm:cusum-hetero} establish non-asymptotic bounds on the localization error (i.e., distance between the estimated change points and their oracle locations) for VCP and WCP, respectively. Crucially, \ul{\textit{VCP's error bound is governed by the maximal variance of the sentence-level scores, whereas the bound for WCP depends on their harmonic mean}}. Since the harmonic mean is strictly bounded by the maximum, these results formally demonstrate that \ul{\textit{WCP outperforms VCP particularly when the detection scores have heterogeneous variances. In the homogeneous case, the two algorithms achieve equivalent performance}}.  
\item \textit{\textbf{Minimax optimality}}: Theorem~\ref{lemma:lowerbound} derives a minimax lower bound for the change point estimation error. Notably, this lower bound matches the upper bound achieved by WCP. Together with Theorem~\ref{thm:cusum-hetero}, this establishes the \ul{\textit{minimax optimality of WCP}}. 
\vspace{-0.5em}
\end{itemize}

\textbf{Non-asymptotic error bound}. Suppose there are $K$ change points $\{\tau_j\}_{j\in [K]}$ in a paragraph $\bm{X}=(X_1,\dotsc,X_N)$. Let $\Delta_1 = \min_{j\in[K+1]}\{\tau_j-\tau_{j-1}\}$ denote the minimum gap between these change points. Furthermore, let $\mu_h = \mathbb{E}[\phi(X_h)]$ and $\mu_m = \mathbb{E}[\phi(X_m)]$ denote the expected scores for a human-authored sentence $X_h$ and an LLM-authored sentence $X_m$, respectively. Let $\sigma_i$ denote the sub-Gaussian parameter of $\phi(X_i)$, defined in \eqref{eq:subG}; we simply refer to $\sigma_i^2$ as the variance proxy of $\phi(X_i)$ in the main text. The following theorem upper bounds the localization error of VCP.
\begin{theorem}[Error bound for VCP]\label{thm:cusum} Under the assumptions and appropriate choice of tuning parameters specified in Appendix~\ref{sec:asp}, if 
there exist some absolute constant $c>0$ and some $0<\delta<1$ such that
\begin{equation}\label{eq:snr-thm-1}
(\mu_m-\mu_h)^2 \Delta_1 \ge c \sigma_{\max}^2\log(N/\delta),
\end{equation}
where $\sigma_{\max}^2 = \max_{i\in [N]}\sigma_i^2$,
then with probability at least $1-\delta$, the outputs of VCP $\{\widehat{\tau}_j\}_{j\in [\widehat{K}]}$ and $\widehat{K}$ satisfy that
\begin{equation}\label{eq:thm1-localisation}
\widehat{K} = K\quad  \text{and }\,\max_{j\in [K]}|\widehat{\tau}_j - \tau_j| =O\Big( \frac{\sigma_{\max}^2\log(N/\delta)}{(\mu_m-\mu_h)^2}\Big).
\end{equation}
\end{theorem}

\Cref{thm:cusum} shows that with high probability, VCP recovers the true number of change points and upper bounds the localization error as specified in \eqref{eq:thm1-localisation}. However, the upper bound depends on the maximal variance $\sigma_{\max}^2$, which is quite pessimistic. In practice, if even a single score $\phi(X_i)$ exhibits high variance (for example, from a short sentence in sentence-level detection), the resulting bound becomes considerably large. This suggests that VCP's theoretical performance is limited by the least reliable sentence classification in the paragraph. 

The condition in \eqref{eq:snr-thm-1} is called the signal-to-noise ratio (SNR) condition, which requires that the detectable signal, represented by the left-hand side (LHS) of \eqref{eq:snr-thm-1}, dominates the noise level on the right-hand side (RHS). The signal strength in this context is a product of two components: (a) $(\mu_m-\mu_h)^2$, which measures how well $\phi$ can distinguish human- and LLM-authored text, and (b) $\Delta_1$, the minimal distance between two consecutive change points. 

We next demonstrate how WCP improves both the signal-to-noise condition \eqref{eq:snr-thm-1} and the error bound \eqref{eq:thm1-localisation} in the following theorem. 

\begin{theorem}[Error bound for WCP]\label{thm:cusum-hetero} 
Under the same assumptions specified in Appendix~\ref{sec:asp}, together with the signal-to-noise condition in \eqref{eq:snr-thm-1}, if we choose $w_i = \sigma_i^{-2}$ and $B(s,e) = e-s$ in the WCP algorithm, its outputs, $\{\widehat{\tau}_j\}_{j\in [\widehat{K}]}$ and $\widehat{K}$ satisfy that, with probability at least $1-\delta$,
\begin{equation}\label{eq:weighted-localisation}
\widehat{K} = K\quad \text{and } \,
\sum_{i\in (\tau_j,\widehat{\tau}_j]\cup (\widehat{\tau}_j,\tau_j]}\frac{1}{\sigma_i^2}=O\Big( \frac{\log(N/\delta)}{(\mu_m-\mu_h)^2}\Big), \; \forall j.
\end{equation}
Moreover, let $\Delta_2 = \min_{j\in[K+1]}\sum_{i = \tau_{j-1}+1}^{\tau_j}\sigma_i^{-2}$. If  $\max_{i\in[N]} \sigma_i^{-2}=O(\Delta_2)$, $M\ge N^2\log(K/\delta)$ and $B(s,e) = S_{s:e}^w$, then the SNR condition in \eqref{eq:snr-thm-1} can be replaced by 
\begin{equation}\label{eq:snr-thm-2}
(\mu_m-\mu_h)^2 \Delta_2 \ge c\log(N/\delta),
\end{equation}
for some absolute constant $c>0$, while the results in \eqref{eq:weighted-localisation} still hold.     
\end{theorem}

The error bound in \eqref{eq:weighted-localisation} measures the weighted distance between the estimated and true change points. To provide intuition, consider sentence-level detection where the variance $\sigma_i^2$ is inversely proportional to the sentence length $n_i$ (i.e., $\sigma_i^2 \propto n_i^{-1}$). In this scenario, the term $\sum \sigma_i^{-2}$ in \eqref{eq:weighted-localisation} represents the total number of tokens between $\tau_j$ and $\widehat{\tau}_j$. We also note that when $K=0$, it follows immediately from the proof (by considering the event $A$) that with an appropriately chosen threshold value $r$, the output satisfies $\hat{K} = 0$ with probability at least $1-\delta$.

Comparing \Cref{thm:cusum} with \Cref{thm:cusum-hetero}, we see that WCP outperforms VCP in two aspects. First, under the same SNR condition \eqref{eq:snr-thm-1}, WCP achieves a smaller localization error. This is because 
\[
\frac{|\tau_j-\widehat{\tau}_j|}{\sigma_{\max}^2}\leq \sum_{i\in (\tau_j,\widehat{\tau}_j]\cup (\widehat{\tau}_j,\tau_j]}\frac{1}{\sigma_i^2},
\]
and hence \eqref{eq:weighted-localisation} implies \eqref{eq:thm1-localisation}. Second, WCP obtains the stronger guarantee even under a weaker SNR condition \eqref{eq:snr-thm-2}, allowing more subtle changes in the signal to be detected. Specifically, the requirement in \eqref{eq:snr-thm-2} is less stringent than that of \eqref{eq:snr-thm-1} because 
\[
\Delta_2 = \min_{j\in[K+1]}\sum_{i = \tau_{j-1}+1}^{\tau_j}\frac{1}{\sigma_i^2} \geq \frac{\Delta_1}{\sigma_{\max}^2},
\] 
which offers a stronger signal on the left-hand side. The proof of \eqref{eq:weighted-localisation} under the weaker condition \eqref{eq:snr-thm-2} relies on refined theoretical arguments, which constitute one of our technical contributions.

\textbf{Minimax optimality}. We next show that the localization error achieved by WCP is minimax optimal. To present the minimax lower bound, we introduce the following notation. For a given $\delta\in (0,1/2)$, an estimator $\widehat{\tau}$ and a distribution $P$, let 
\[
Q(\delta,\widehat{\tau},P) := \inf\{\eta\in[0,\infty): P(|\widehat{\tau}-\tau|\leq \eta) \geq 1-\delta\},
\]
which can be interpreted as the smallest localisation error that can be achieved by $\widehat{\tau}$ with probability at least $1-\delta$.
\begin{theorem}[Minimax lower bound]\label{lemma:lowerbound}
Suppose Assumptions (i) and (ii) in Appendix~\ref{sec:asp} hold. Suppose that there is only one change point $1<\tau< N$ representing a transition from LLM-authored to human-authored text. Let $P$ denote the joint distribution of $\{\phi(X_i)\}_{i\in [N]}$. Consider the class of distributions
\begin{equation}\label{eq:one-change-class}
\mathcal{P} = \Big\{P: (\mu_m-\mu_h)^2\min\{S^w_{1:\tau},S^w_{(\tau+1):N}\}\geq C\log(1/\delta)\Big\},
\end{equation}
with $w_i = \sigma_i^{-2}$, for some absolute constant $C>0$. It holds that 
\[
\inf_{\widehat{\tau}} \sup_{P\in \mathcal{P}}Q(\delta,\widehat{\tau},P) \geq \max\{h_1,h_2\},
\]
where $h_1$ and $h_2$ satisfy 
\begin{equation}\label{eq:lowerbound-solution}
\sum_{{\tau+1}}^{\tau+h_1}\frac{1}{\sigma_i^2} =  \frac{c\log(1/\delta)}{(\mu_m-\mu_h)^2},\quad \sum_{{\tau-h_2}}^{\tau}\frac{1}{\sigma_i^2} =  \frac{c\log(1/\delta)}{(\mu_m-\mu_h)^2},
\end{equation}
for some constant $c>0$.
\end{theorem}

To see that the guarantee in \Cref{thm:cusum-hetero} is minimax optimal, we set the failure probability $\delta = N^{-1}$, which is a natural choice to ensure that \eqref{eq:weighted-localisation} holds with probability approaching $1$ as $N \rightarrow \infty$. Without loss of generality, we consider the upper bound in the case of a single change point. When $\widehat{\tau} > \tau$, \eqref{eq:weighted-localisation} becomes
\[
\sum_{i=\tau+1}^{\widehat{\tau}}\frac{1}{\sigma_i^2}=O\Big( \frac{\log(N)}{(\mu_m-\mu_h)^2}\Big),
\]
which together with the definition of $h_1$ in \eqref{eq:lowerbound-solution}, immediately implies $\widehat{\tau}-\tau =O(h_1)$. Similarly, we have $\tau-\widehat{\tau} =O( h_2)$ when $\tau > \widehat{\tau}$. Combining these two cases, we conclude that WCP's estimated change point satisfies
\[
|\widehat{\tau} - \tau| =O( \max\{h_1,h_2\}),
\]
and therefore matches the lower bound up to constants in \Cref{lemma:lowerbound}. We present here the lower bound concerning localizing a single change for notation simplicity, as it is commonly done in the literature \cite[e.g.][]{wang2020univariate,wang2021optimal,li2022network}. Additionally, with minor changes in notation, this lower bound also immediately applies to any larger model class $\mathcal{P}^*$, including the one that contains at most $K$ change points, as long as the model class in \eqref{eq:one-change-class} satisfies $\mathcal{P}\subseteq\mathcal{P}^*$. Also, we clarify that our lower bound should be interpreted conditional on a fixed detector $\phi$: after each text segment $X_i$ is transformed into the statistic $\phi(X_i)$, no estimator can achieve a worst-case localization error smaller than the stated lower bound.

\textbf{Extension to dependent scores and unknown score variances.} The results so far require the scores $\phi(X_i)$ to be independent and $\sigma_i$ to be known. We now extend our theory to relax these requirements. We consider the following two types of conditions: (i) Suppose that $\sigma_i^2$ are unknown but we have access to estimators $\hat{\sigma}_i^2$ such that with probability at least $1-\delta$,  for some $R \in (0,1)$
\begin{equation*}
\max_{i = 1,\dotsc,N}|\hat{\sigma}^2_i/\sigma_i^2 - 1|\leq R.
\end{equation*}
We note that the above condition implies $(1+R)^{-1}\sigma_i^{-2} \leq \hat{\sigma}_i^{-2}\leq (1-R)^{-1}\sigma_i^{-2}$. Our results allow $R$ to be an absolute constant, which essentially only requires estimating the variances up to some constant factor and is much weaker than requiring consistent estimation of the unknown $\sigma_i^{-2}$. (ii) We can also allow temporal dependence measured by the total-variation dependence measure considered in \cite[e.g.][]{samson2000concentration,kontorovich2008concentration}. Specifically, let $Z_i = (\phi(X_i) - \mathbb{E}(\phi(X_i))/\sigma_i$ and $\mathcal{L}(Z\mid Y=y)$ denote the conditional distribution of $Z$ given $Y=y$, and for $1\leq i < j \leq N$, let 
\[
{\eta}_{ij}
=
\sup_{z_{1:i-1},\,z,z'}
\left\|
\mathcal{L}\!\left(
Z_{j:N}\mid Z_{1:i-1}=z_{1:i-1},\, Z_i=z
\right)
-
\mathcal{L}\!\left(
Z_{j:N}\mid Z_{1:i-1}=z_{1:i-1},\, Z_i=z'
\right)
\right\|_{\mathrm{TV}}.
\]
Assume that ${\eta}_{ij} \leq \rho(j-i)$ for some $\rho$ and $\Lambda_{\rho}
=
1+\sum_{k=1}^{\infty}\sqrt{\rho(k)}
<\infty.$ We choose this type of dependency measure due to the availability of sharp concentration inequality \cite{samson2000concentration} that is compatible with our heterogeneous noise setup. Specific processes that satisfy this type of condition can be found in \cite[e.g.][]{samson2000concentration,kontorovich2008concentration}. The detailed description of these two conditions is collected in \Cref{asp:extension} in \Cref{sec:extension}, under which we have the following guarantee.

\begin{theorem}\label{thm:extension}
Suppose \Cref{asp:extension} holds.
let $\Delta_2 = \min_{j\in[K+1]}\sum_{i = \tau_{j-1}+1}^{\tau_j}\sigma_i^{-2}$, and $\Gamma_R = (1+R)/(1-R)$. If  $\max_{i\in[N]} \sigma_i^{-2}=O(\Gamma_R^{-1}\Delta_2)$, $M\ge N^2\log(K/\delta)$, and
\begin{equation}\label{eq:snr-2}
(\mu_m-\mu_h)^2 \Delta_2 \ge c\Gamma_R\Lambda_\rho^2\log(N/\delta),
\end{equation}
for some absolute constant $c>0$,
then 
if we choose $w_i = \hat{\sigma}_i^{-2}$, $B(s,e) = S^w_{s:e}$, and $r = O(\Lambda_\rho\sqrt{\frac{\log(N/\delta)}{1-R}})$ in the WCP algorithm, its outputs, $\{\widehat{\tau}_j\}_{j\in [\widehat{K}]}$ and $\widehat{K}$ satisfy that, with probability at least $1-\delta$,
\begin{equation}
\widehat{K} = K\quad  \text{and } \,
\sum_{i\in (\tau_j,\widehat{\tau}_j]\cup (\widehat{\tau}_j,\tau_j]}\frac{1}{\sigma_i^2}=O\Big( \Gamma_R\Lambda_\rho^2\frac{\log(N/\delta)}{(\mu_m-\mu_h)^2}\Big), \; \forall j.
\end{equation}
\end{theorem}
We note that when $R\asymp \Lambda_\rho\asymp 1$, then the results in \Cref{thm:extension} agree with that of \Cref{thm:cusum-hetero}, which implies the WCP algorithm maintains optimality even in situations where there is temporal dependence and the oracle weights are unknown.

\section{Experiments}\label{sec:exp}
\vspace{-7pt}
We conduct extensive experiments to evaluate the empirical performance of our approach. In this section, we focus on sentence-level detection. We first present a semi-synthetic analysis that constructs synthetic data from real-world human-written benchmark datasets and consider settings with a single change point. We then conduct a real data analysis to illustrate the practical utility of our approach. 
In Appendix~\ref{sec:simu-multi-cp}, we extend our evaluation to settings with multiple change points. Appendix~\ref{sec:simu-sensitive} further conducts a sensitivity analysis under more complex settings to demonstrate the robustness of our method. Appendix~\ref{sec:token-experiments} studies token-level detection. 

\textbf{Semi-synthetic analyses}. We first describe our procedure for generating synthetic human--LLM co-authored text. We randomly sample 100 human-written documents from each of three publicly available benchmark datasets for LLM-generated text detection: \textit{WikiQA} \citep{rajpurkar2016squad}, \textit{News} \citep{Narayan2018DontGM}, and \textit{Story} \citep{fan2018hierarchical}. For each document, we split the text into multiple segments. We then prompt LLMs to regenerate some segments while keeping the remaining segments in their original human-authored form. To ensure diversity, several models are employed, including GPT-5-mini \citep{singh2025openai} and Claude 4.5 \citep{anthropic2025claude4.5}.  Specific prompts used for LLM rewriting are detailed in Appendix~\ref{sec:exp-details}. 

We compare VCP, WCP, a sentence-level prediction baseline \citep[denoted by \underline{SenPred},][]{kushnareva2024boundary}, a majority-voting algorithm \citep[denoted by \underline{Voting},][]{zhang2024machine}, a prompting method that directly prompts the LLM that rewrites the paragraph to segment the human--LLM co-authored text (denoted by \underline{LLMPred}), \underline{TextTiling} \citep{hearst1997text}, and an existing LLM detector \citep[\underline{PaLD},][]{lei2025pald}. To ensure fairness among score-based methods, we use the same detection score, $\phi$, for segmentation. Our main text presents results obtained using AdaDetectGPT \citep{zhou2025adadetectgpt}. We also consider several choices of $\phi$ in our sensitivity analysis (Appendix~\ref{sec:simu-sensitive}). 

For each method, we use the WindowDiff (WD) metric~\citep{pevzner2002critique} to evaluate the accuracy of its estimated change points. We also report count error (CE), which denotes the difference between the true number of change points and the estimated number. A negative CE indicates overestimation of the number of change points. We report the average WD and CE across 100 documents for every combination of LLM generator and dataset. We first evaluate the performance of our methods under a single change point setting. When generating the co-authored text, the change-point location is sampled at random for each document. This setting better reflects practical cases, where the LLM-written segment may appear at an arbitrary position.

\begin{wraptable}{r}{0.70\linewidth}
\centering
\caption{Results on single change-point detection. The best results are presented in \textbf{bold}. PaLD has a high computational cost; thus, we only consider its results on the WikiQA datasets. }\label{tab:single-cp}
\vspace*{-6pt}
\setlength{\tabcolsep}{1.8pt}
\small
\begin{tabular}{@{}lcccccc@{}}
\toprule
\multirow{2}{*}{Method} & \multicolumn{2}{c}{WikiQA} & \multicolumn{2}{c}{News} & \multicolumn{2}{c}{Story} \\
\cmidrule(lr){2-3}\cmidrule(lr){4-5}\cmidrule(lr){6-7}
& WD & CE & WD & CE & WD & CE \\
\midrule
LLMPred & 0.49{\scriptsize$\pm$0.02} & -0.53{\scriptsize$\pm$0.18} & 0.60{\scriptsize$\pm$0.02} & -6.26{\scriptsize$\pm$0.72} & 0.59{\scriptsize$\pm$0.03} & -6.87{\scriptsize$\pm$0.82} \\
TextTiling & 0.29{\scriptsize$\pm$0.01} & 1.00{\scriptsize$\pm$0.00} & 0.31{\scriptsize$\pm$0.00} & 1.00{\scriptsize$\pm$0.00} & 0.31{\scriptsize$\pm$0.00} & \textbf{1.00{\scriptsize$\pm$0.00}} \\
SenPred & 0.24{\scriptsize$\pm$0.02} & -0.91{\scriptsize$\pm$0.11} & 0.70{\scriptsize$\pm$0.02} & -6.84{\scriptsize$\pm$0.31} & 0.91{\scriptsize$\pm$0.01} & -12.49{\scriptsize$\pm$0.44} \\
PaLD & 0.49{\scriptsize$\pm$0.04} & -1.53{\scriptsize$\pm$0.27} & -- & -- & -- & -- \\
Voting & 0.29{\scriptsize$\pm$0.03} & \textbf{0.00{\scriptsize$\pm$0.04}} & \textbf{0.10{\scriptsize$\pm$0.01}} & \textbf{-0.34{\scriptsize$\pm$0.09}} & 0.42{\scriptsize$\pm$0.03} & -2.41{\scriptsize$\pm$0.22} \\
VCP & 0.20{\scriptsize$\pm$0.02} & 0.22{\scriptsize$\pm$0.07} & \textbf{0.10{\scriptsize$\pm$0.02}} & -1.09{\scriptsize$\pm$0.16} & \textbf{0.28{\scriptsize$\pm$0.02}} & -1.02{\scriptsize$\pm$0.16} \\
WCP & \textbf{0.19{\scriptsize$\pm$0.02}} & 0.01{\scriptsize$\pm$0.09} & \textbf{0.10{\scriptsize$\pm$0.02}} & -1.69{\scriptsize$\pm$0.19} & \textbf{0.28{\scriptsize$\pm$0.02}} & -2.60{\scriptsize$\pm$0.22} \\

\bottomrule
\end{tabular}
\end{wraptable}
Table~\ref{tab:single-cp} reports the results under the single random change-point setting. Since PaLD is computationally demanding, we implement it only on the WikiQA dataset. 
Table~\ref{tab:single-cp} leads to the following observations: (i) WCP and VCP attain the smallest WD across all three datasets. The improvement over the baselines is particularly evident on the WikiQA and Story datasets, where sentence-wise prediction and direct LLM prompting produce substantially larger segmentation errors. This demonstrates the effectiveness of using change point detection to segment human–LLM co-authored text. (ii) The baseline methods, SenPred, LLMPred, and PaLD, often produce large negative CEs, indicating a tendency to overestimate the number of change points. This overestimation, in turn, leads to larger WDs. (iii) The TextTiling method uses semantic similarity to segment text, but its segmentation may not align with the boundaries created by humans and LLMs.

\begin{wraptable}{r}{0.57\linewidth}
\centering
\vspace{-12pt}
\caption{WD results under paragraph-length settings.}\label{tab:vcp-wcp}
\vspace{-7pt}
\setlength{\tabcolsep}{2pt}
\small
\begin{tabular}{l|ccccc}
\toprule
$k$ & 2 & 3 & 4 & 5 & 6 \\
\midrule
Voting & 0.49{\scriptsize$\pm$0.07} & 0.51{\scriptsize$\pm$0.08} & 0.49{\scriptsize$\pm$0.09} & 0.50{\scriptsize$\pm$0.10} & 0.61{\scriptsize$\pm$0.05} \\
SenPred & 0.74{\scriptsize$\pm$0.06} & 0.75{\scriptsize$\pm$0.07} & 0.78{\scriptsize$\pm$0.06} & 0.72{\scriptsize$\pm$0.08} & 0.71{\scriptsize$\pm$0.06} \\
TextTiling & 0.41{\scriptsize$\pm$0.07} & 0.52{\scriptsize$\pm$0.09} & 0.48{\scriptsize$\pm$0.05} & 0.55{\scriptsize$\pm$0.08} & 0.76{\scriptsize$\pm$0.05} \\
VCP & 0.43{\scriptsize$\pm$0.07} & 0.41{\scriptsize$\pm$0.07} & 0.47{\scriptsize$\pm$0.05} & 0.47{\scriptsize$\pm$0.04} & 0.51{\scriptsize$\pm$0.04} \\
WCP & \textbf{0.40{\scriptsize$\pm$0.08}} & \textbf{0.37{\scriptsize$\pm$0.08}} & \textbf{0.32{\scriptsize$\pm$0.08}} & \textbf{0.29{\scriptsize$\pm$0.09}} & \textbf{0.28{\scriptsize$\pm$0.08}} \\
\bottomrule
\end{tabular}
\end{wraptable}

\textbf{Paragraph-level detection}. While the theoretical advantage of WCP over VCP is clear from our results in \Cref{sec:theory}, in practice, WCP is expected to outperform VCP when segment lengths vary substantially. We illustrate this using a synthetic paragraph-level change-point task with a single change point. We consider settings where each paragraph contains $j$ sentences, with $j \in \{\textup{FN}_1, \ldots, \textup{FN}_k\}$, where $\textup{FN}_k$ denotes the $k$-th Fibonacci number. For example, when $k=2$, we have $\{\textup{FN}_1, \ldots, \textup{FN}_k\} = \{1, 1\}$, so all paragraphs contain only one sentence. In contrast, when $k=6$, we have $\{\textup{FN}_1, \ldots, \textup{FN}_k\} = \{1, 1, 2, 3, 5, 8\}$, leading to substantial variation in the number of sentences across paragraphs. The empirical results for $k$ ranging from 2 to 6 are summarized in Table~\ref{tab:vcp-wcp}. We observe that WCP achieves better performance than VCP, and that this advantage becomes more pronounced as $k$ increases.

\textbf{Real data analyses}. We next evaluate our method on a real-world human--AI co-authored text dataset provided by \citet{lee2022coauthor}, referred to as the CoAuthor dataset. Following the setup in \citet{zeng2024towards}, each document is annotated at the sentence level with three labels: human-written, collaboratively written, and fully LLM-generated sentences. This setting turns localization into a three-class segmentation problem, which better reflects practical human--AI co-authoring scenarios. In addition to the baselines in the semi-synthetic data analyses, we also compared a supervised learning-based algorithm considered in \citet{zeng2024towards}, SegFormer \citep{bai2023segformer}. Results are summarized in Table~\ref{tab:real-data-coauthor}. VCP and WCP achieve the smallest WD, indicating better change localization on this real-world co-authoring dataset. SegFormer obtains the smallest absolute CE, which is expected for a supervised model trained on the dataset. Nonetheless, its WD remains larger than that of the proposed change-point methods. Unsupervised baselines, such as Voting and TextTiling, have substantially larger WD, suggesting that they remain less effective at handling segmentation tasks in this three-class setting.

\section{Conclusion}
This paper proposes a change point detection approach to localize human- and LLM-authored text in hybrid human--LLM documents. We develop  algorithms to address this largely unexplored problem, establish their finite-sample error guarantees and minimax optimality, and demonstrate their superior empirical performance. By bridging change point detection with modern LLMs, our work paves the way for leveraging classical time-series tools to solve the evolving challenges posed by LLM-generated content.

\begin{ack}
The authors thank the three anonymous referees and the area chair for their insightful and constructive comments.

Mengchu Li, Jin Zhu, and JingLai Li were supported by the AIRR Gateway project 2026, ``Efficient Approaches on Detecting Generative AI''. Shi's work was supported by the AIRR Gateway Project 2026, ``Reasoning-Enhanced Multimodal AI-Generated Content Detection''.
\end{ack}

\bibliographystyle{plainnat}
\bibliography{reference}

\appendix
\section{Additional literature review}\label{sec:app-literature}

\textbf{LLM-generated text detection}. These works primarily study the binary classification problem of distinguishing between human- and LLM-authored text. Existing approaches can be broadly categorized into machine learning (ML)-based and zero-shot methods. Specifically, ML-based methods collect datasets of human-written text, prompt LLMs to produce corresponding rewrites, and then utilize both sources of text to train a classifier to perform this binary classification \citep{solaiman2019release, ippolito2020bert, guo2023close, hu2023radar, mao2024raidar, guo2024biscope, tian2024multiscale, yu2024dpic, chen2025imitate, zhou2025adadetectgpt, zhou2026learntodistance, su2026reader}. In contrast, zero-shot methods are training-free. They construct statistical measures derived directly from the text, such as its intrinsic dimensionality, or based on the next-token probability distributions of a target LLM \citep{gehrmann2019gltr, mitchell2023detectgpt, su2023detectllm, tulchinskii2023intrinsic, bao2024fastdetectgpt, hans2024spotting, yang2024dnagpt, song2025deep}. These measures take on different values for human and LLM-authored content, which serves as the basis for classification. 

While our proposal employs similar techniques for distinguishing between human and LLM-generated text, we study the more complex problem of text localization rather than providing a binary label for the entire passage, as discussed below.

\textbf{Change point detection}. Change point detection aims to identify and localize structural changes within a sequence of observations -- for instance, a sudden shift in their expected value. It has been a popular research topic in statistics and machine learning \citep[e.g.,][]{baranowski2019narrowest,verzelen2023optimal,bhattacharyya2025theoretical,wang2020univariate}. 
The literature on change point detection covers different types of data and various categories of structural changes. Methodologically, the CUSUM statistic, introduced in \Cref{sec:preliminary}, plays a central role in numerous algorithms \citep[e.g.][]{fryzlewicz2014wild,baranowski2019narrowest,wang2020univariate,cho2022two,wang2021optimal}. Theoretically, minimax rates of convergence have been established to understand the fundamental difficulty of different problems and to provide a rigorous benchmark for evaluating different algorithms \citep[e.g.][]{yu2020review}.

\section{Proofs}

\subsection{Assumptions}\label{sec:asp}
We list the assumptions required for \Cref{thm:cusum} and \Cref{thm:cusum-hetero} to hold here. The $\psi_2$-norm for sub-Gaussian random variables is defined in \eqref{eq:subG}. 
\begin{enumerate}
\item[(i)] $\phi(X_1),\phi(X_2),\cdots,\phi(X_N)$ are independent;\vspace{-0.5em}
\item[(ii)] $\|\phi(X_i) - \mathbb{E}(\phi(X_i))\|_{\psi_2} \leq \sigma_i$ for all $i\in [N]$;\vspace{-0.5em}
\item[(iii)] the input threshold $r$ in Algorithm \ref{alg:not_algorithm} is set to be proportional to $\sqrt{\log(N/\delta)}$;\vspace{-0.5em}
\item[(iv)] there exists some sufficiently large constant $C>0$ such that the number of random intervals $M \ge C N^2\Delta_1^{-2}\log(N(\delta\Delta_1)^{-1})$.\vspace{-1em}
\end{enumerate}
\vspace{1em}

We note that the independence and sub-Gaussianity assumptions, (i) and (ii), are primarily imposed to simplify the theoretical analysis and are standard in the change point detection literature \citep[e.g.,][]{wang2020univariate, baranowski2019narrowest, verzelen2023optimal}. We extend the results to accommodate temporal dependence among the scores in \Cref{thm:extension} with details in \Cref{sec:extension}. It is also possible to relax the sub-Gaussianity to heavy-tailed distributions \citep[e.g.,][]{li2021adversarially,cho2022two,li2025robust}. Assumptions (iii) and (iv) are mild as both $r$ and $M$ are user-specified. 

\subsection{Proofs for Theorems~\ref{thm:cusum}--~\ref{lemma:lowerbound}}

For a real-valued random variable $X$, we define its Orlicz-$\psi_2$ norm as 
\begin{equation}\label{eq:subG}
\|X\|_{\psi_2} = \inf \{t > 0: \mathbb{E}[\exp(\{|X|/t\}^2)]\leq 2\}.
\end{equation} A variable with a finite Orlicz-$\psi_2$ norm is sub-Gaussian, meaning its distribution has a tail that is at least as light as that of a Gaussian variable.

In this section, we collect the proofs of results in the main sections. To simplify notation, we shall use $\bm{Y} = (Y_1,\dotsc,Y_N)$ to denote $\phi(\bm{X}) = (\phi(X_1),\dotsc,\phi(X_N))$, and use $\mu = (\mu_1,\dotsc,\mu_N)$ to denote the mean vector of $\bm{Y}$. We can further write $Y_i = \mu_i+\epsilon_i$ where each $\epsilon_i$ satisfies $\|\epsilon_i\|_{\psi_2} \leq \sigma_i.$
We also introduce the following general notation. Recall that for a general weight vector ${w} = (w_1,\dotsc,w_N)$, $S^w_{u:v} = \sum_{i=u}^vw_i,\; 1\leq u\leq v\leq N$. For vectors ${v}, w \in \R^N$, and $b \in [N]$, with $1\leq s\leq b\leq e-1 < N$, we write 
\begin{equation}\label{eq:vec-cusum}
W_{s,e}^{\bm{v}}(b) = \sqrt{\frac{S^w_{s:b} \,S^w_{(b+1):e}}{S^w_{s:e}}}\Bigg|\sum_{i=s}^b {\frac{w_i}{S^w_{s:b}}}v_i
- \sum_{i=b+1}^{e}{\frac{w_i}{S^w_{(b+1):e}}}v_i\Bigg| = |\langle v,\psi_{s,e}^{b}\rangle_{w}|,
\end{equation}
where 
\[
\psi_{s,e}^b(i) = \begin{cases}
\sqrt{\frac{S^w_{(b+1):e}}{S^w_{s:b}\, S^w_{s:e}}},  &s\leq i \leq b,\\
-\sqrt{\frac{S^w_{s:b} }{S^w_{s:e}\,S^w_{(b+1):e}}}, &b< i\leq e,\\
0,  &\text{otherwise},
\end{cases}
\]
and $\langle x,y\rangle_{w} = \sum_{i=1}^N x_iy_iw_i$ for two vectors $x,y \in \R^N$. We also denote $\|x\|_{w}^2 = \langle x,x \rangle_w = x^\top Wx$, with $W = \text{diag}(w).$ 

The following lemma is important in the proof of \Cref{thm:cusum-hetero}.

\begin{lemma}\label{lemma:cusum-onechange}
Let $\bmu = (\mu_1,\dotsc,\mu_N)$ denote the mean vector of $\bm{Y} = (Y_1,\dotsc,Y_N)$ and let $\tau_1,\dotsc,\tau_K$ be the change points. Suppose $1\leq s < e\leq N$, such that $\tau_{j-1} < s \leq \tau_j < e \leq \tau_{j+1}$ for some $j = 1,\dotsc,K$. Let $\eta = \min\{S^w_{s:\tau_j},S^w_{(\tau_j+1):e}\}$ and $\kappa = |\mu_h-\mu_m|$. Then 
\[
W_{s,e}^{\bmu}(\tau_j) = \max_{s\leq b<e}W_{s,e}^{\bmu}(b),
\]
and $\sqrt{\eta}\kappa/\sqrt{2}\leq W_{s,e}^{\bmu}(\tau_j)\leq \sqrt{\eta}\kappa$. Moreover, it holds that 
\begin{equation}\label{eq:identity}
\|\psi_{s,e}^b\langle \psi_{s,e}^b,\bmu \rangle_w - \psi_{s,e}^{\tau_j}\langle \psi_{s,e}^{\tau_j},\bmu \rangle_w\|_w^2 = (W_{s,e}^{\bm{\mu}}(\tau_j))^2- (W_{s,e}^{\bm{\mu}}(b))^2,
\end{equation}
and 
\begin{equation}\label{eq:cusum-squarediff}
(W_{s,e}^{\bm{\mu}}(\tau_j))^2- (W_{s,e}^{\bm{\mu}}(b))^2 =
\begin{cases}
\frac{S^w_{(b+1):\tau_j}S^w_{(\tau_j+1):e}}{S^w_{(b+1):e}}\kappa^2 & \text{if}\;  s\leq b<\tau_j, \\
\frac{S^w_{(\tau_j+1):b}S^w_{s:\tau_j}}{S^w_{s:b}}\kappa^2 & \text{if}\;  \tau_j\leq b < e.
\end{cases}
\end{equation}
\end{lemma}
\begin{proof}[Proof of \Cref{lemma:cusum-onechange}] For a general triplet $s\leq b< e$, it holds that  
\begin{equation}\label{eq:cusum-representation}
W_{s,e}^{\bm{\mu}}(b) = \begin{cases}
\kappa\sqrt{\frac{S^w_{s:b}}{S^w_{s:e}S^w_{(b+1):e}}}S^w_{(\tau_j+1):e}, & b \leq \tau_j; \\\kappa
\sqrt{\frac{S^w_{s:\tau_j}S^w_{(\tau_j+1):e}}{S^w_{s:e}}}, & b = \tau_j; \\
\kappa\sqrt{\frac{S^w_{(b+1):e}}{S^w_{s:e}S^w_{s:b}}}S^w_{s:\tau_j}, & b \geq \tau_j,
\end{cases}
\end{equation}
and then it is clear that $W_{s,e}^{\bm{\mu}}(b)$ is maximized at $b = \tau_j$ since all $w_i \geq 0$ for $i\in [N]$. Note that 
\[
\sqrt{\eta/2}\leq \sqrt{\frac{S^w_{s:\tau_j}S^w_{(\tau_j+1):e}}{S^w_{s:e}}} \leq \sqrt{\eta},
\]
which leads to the fact that $\sqrt{\eta}\kappa/\sqrt{2}\leq W_{s,e}^{\bmu}(\tau_j)\leq \sqrt{\eta}\kappa$. The expression for $(W_{s,e}^{\bm{\mu}}(\tau_j))^2- (W_{s,e}^{\bm{\mu}}(b))^2$ in \eqref{eq:cusum-squarediff} follows from direct calculation using \eqref{eq:cusum-representation}.

Finally, to see \eqref{eq:identity}, we consider first the case $b < \tau_j$, then 
\begin{equation}\label{eq:lemmaA1}
\|\psi_{s,e}^b\langle \psi_{s,e}^b,\bmu \rangle_w - \psi_{s,e}^{\tau_j}\langle \psi_{s,e}^{\tau_j},\bmu \rangle_w\|_w^2 = \sum_{i=s}^ew_ik_i^2
\end{equation}
where 
\[
k_i = \begin{cases}
\sqrt{\frac{S^w_{(b+1):e}}{S^w_{s:b}\, S^w_{s:e}}}\langle \psi_{s,e}^{b},\bmu \rangle_w - \sqrt{\frac{S^w_{(\tau_j+1):e}}{S^w_{s:\tau_j}\, S^w_{s:e}}}\langle \psi_{s,e}^{\tau_j},\bmu \rangle_w, & s\leq i\leq b; \\
-\sqrt{\frac{S^w_{s:b} }{S^w_{s:e}\,S^w_{(b+1):e}}}\langle \psi_{s,e}^{b},\bmu \rangle_w - \sqrt{\frac{S^w_{(\tau_j+1):e}}{S^w_{s:\tau_j}\, S^w_{s:e}}}\langle \psi_{s,e}^{\tau_j},\bmu \rangle_w, & b< i\leq \tau_j;\\
-\sqrt{\frac{S^w_{s:b} }{S^w_{s:e}\,S^w_{(b+1):e}}}\langle \psi_{s,e}^{b},\bmu \rangle_w +\sqrt{\frac{S^w_{s:\tau_j} }{S^w_{s:e}\,S^w_{(\tau_j+1):e}}}\langle \psi_{s,e}^{\tau_j},\bmu \rangle_w, & \tau_j< i\leq e.
\end{cases}
\]
Under the observation \eqref{eq:vec-cusum}, we can compute $\langle \psi_{s,e}^{\tau_j},\bmu \rangle_w$ and $\langle \psi_{s,e}^{b},\bmu \rangle_w$ using \eqref{eq:cusum-representation} which yields
\[
k_i = \begin{cases}
0, & s\leq i\leq b;\\
-(\mu_h-\mu_m)\frac{S^w_{(\tau_j+1):e}}{S^w_{(b+1):e}}, & b< i\leq \tau_j; \\
(\mu_h-\mu_m)\frac{S^w_{(b+1):\tau_j}}{S^w_{(b+1):e}}, & \tau_j< i\leq e.
\end{cases}
\]
Substituting the above in \eqref{eq:lemmaA1} gives
\[
\sum_{i=s}^e w_i k_i^2
=
\frac{S^w_{(b+1):\tau_j}S^w_{(\tau_j+1):e}}{S^w_{(b+1):e}}(\mu_h-\mu_m)^2
=
\frac{S^w_{(b+1):\tau_j}S^w_{(\tau_j+1):e}}{S^w_{(b+1):e}}\kappa^2,
\]
which is the claimed result when $b < \tau_j$. The other case can be obtained in the same way.  
\end{proof}

We first present the proof of \Cref{thm:cusum-hetero} and then comment on how \Cref{thm:cusum} directly follows from it. 
\begin{proof}[Proof of \Cref{thm:cusum-hetero}]

We follow the structure of the proof of Theorem~1 in \citet{baranowski2019narrowest}, which consists of \emph{Five Steps}, and make the necessary changes to accommodate the heterogeneity of $\sigma_i$ and our \textbf{WCP} algorithm. We need to modify the events considered in their \emph{Steps One and Two} as follows. Let $\bm{\mu}_{\psi} = \psi_{s,e}^b\langle \psi_{s,e}^b,\bmu \rangle_w - \psi_{s,e}^{\tau_j}\langle \psi_{s,e}^{\tau_j},\bmu \rangle_w$ and
\[
\mathcal{D}=\{(j,s,e,b): j=1,\dotsc,K,\ \tau_{j-1}<s\leq\tau_j,\ \tau_j<e\leq\tau_{j+1},\ s\leq b<e\}.
\]
Consider the following two events:
\begin{equation}\label{eq:prob-events}
\begin{aligned}
A &= \{\max_{s,b,e:1\leq s\leq b< e\leq N}|W_{s,e}^{\bm{\epsilon}}(b)|\leq r\},\\
B &= \Big\{\frac{|\langle \bm{\mu}_{\psi},\bm{\epsilon}\rangle_w|}{\|\bm{\mu}_{\psi}\|_w}\leq r,\ \forall (j,s,e,b)\in\mathcal{D}\Big\}.
\end{aligned}
\end{equation}
Note that using independence and sub-Gaussianity of the components of $\bm{\epsilon}$ with $\|\epsilon_i\|_{\psi_2}\leq \sigma_i = w_i^{-1/2}$, we have 
\[
\|\langle \bm{\epsilon},\psi_{s,e}^{b}\rangle_{w}\|_{\psi_2}^2 = \Big\|\sum_{i=1}^N{w_i\epsilon}_i\psi_{s,e}^{b}(i)\Big\|_{\psi_2}^2 \leq C \sum_{i=1}^N \Big\|w_i\epsilon_i\psi_{s,e}^{b}(i)\Big\|_{\psi_2}^2 \leq C,
\]
for some absolute constant $C$, where the first inequality is due to Proposition 2.7.1 in \citet{vershynin2018high}. Similarly, it also holds that 
\[
\bigg\|\frac{\langle \bm{\mu}_{\psi},\bm{\epsilon}\rangle_w}{\|\bm{\mu}_{\psi}\|_w}\bigg\|_{\psi_2} \leq C.
\]
Hence, using a union bound and tail bound for sub-Gaussian random variable \citep[e.g.][Proposition 2.6.6]{vershynin2018high}, we obtain 
\[
\PP(A^c) \leq N^3\exp(-cr^2)\quad \PP(B^c) \leq N^4\exp(-cr^2).
\]
Choosing $r = c'\sqrt{\log(N/\delta)}$ with a large enough absolute constant $c'>0$ ensures both $\PP(A^c)$ and $\PP(B^c)$ are less than $\delta$.

\emph{Step Three} is the key step showing the difference between the signal-to-noise conditions \eqref{eq:snr-thm-1} and \eqref{eq:snr-thm-2}. The main goal of this step is to show that when a large number of random intervals are generated, each change point $\tau_j$ is captured by some interval. However, under different conditions, we need to consider intervals at different ``scales.'' 

\textbf{Under the signal-to-noise condition in \eqref{eq:snr-thm-1}}, i.e.\ 
\[
(\mu_m-\mu_h)^2 \Delta_1 \gtrsim \sigma_{\max}^2\log(N/\delta) =\frac{\log(N/\delta)}{\min_{i\in[N]} w_i},
\]
where $\Delta_1 = \min_{j\in[K+1]}\{\tau_j-\tau_{j-1}\}$, we can consider intervals directly at the index scale, as in the proof of Theorem 1 in \citet{baranowski2019narrowest}. Let 
\[
\mathcal{I}^L_j = (\tau_j-\Delta_1/3, \tau_j-\Delta_1/6) \quad \mathcal{I}^R_j = (\tau_j+\Delta_1/6,\tau_j+\Delta_1/3),
\]
and $\mathcal{M} = \{[s_1,e_1],\dotsc,[s_M,e_M]\}$ be the randomly generated intervals. The following event 
\[
C = \{\forall j = 1\dotsc,K,\exists \;m \in \{1,\dotsc,M\}, \text{s.t.}\; s_m\in \mathcal{I}^L_j, e_m \in \mathcal{I}^R_j\},
\]
satisfies 
\[
\PP(C) \geq 1-\delta \quad \text{if} \quad M \geq 36N^2\Delta_1^{-2}\log(N(\delta\Delta_1)^{-1}),
\]
using the arguments in the original proof. 

\textbf{Under the signal-to-noise condition in \eqref{eq:snr-thm-2}}, i.e. 
\[
(\mu_m-\mu_h)^2 \Delta_2 \gtrsim \log(N/\delta),
\]
where $\Delta_2 = \min_{j\in[K+1]}S^w_{(\tau_{j-1}+1):\tau_j} = \min_{j\in[K+1]}\sum_{i = \tau_{j-1}+1}^{\tau_j}\sigma_i^{-2}$, we need to adapt the lengths of $\mathcal{I}^L_j$ and $\mathcal{I}^R_j$ to the scale of the weights. Specifically, with a slight abuse of notation we again let 
\begin{equation}\label{eq:interval-adjustment}
\begin{aligned}
\mathcal{I}^L_j &= \{s \in\{1,\dotsc,\tau_j\}:\Delta_2/6<S^w_{s:\tau_j}< \Delta_2/3\},\\
\mathcal{I}^R_j &= \{e \in\{\tau_j+1,\dotsc,N\}:\Delta_2/6<S^w_{(\tau_j+1):e}< \Delta_2/3\}.
\end{aligned}
\end{equation}
Note that the set $\mathcal{I}^L_j $ will not include any points $i \leq \tau_{j-1}$ since $S^w_{(\tau_{j-1}+1):\tau_j}\geq \Delta_2$. Similarly, $\mathcal{I}^R_j$ will also not include any points $i > \tau_{j+1}$. Moreover, these two sets are non-empty since $\max_{i\in[N]} w_i < \Delta_2/6$ by assumption, so that each set contains at least one element. 
Again, we let $[s_1,e_1],\dotsc,[s_M,e_M]$ be the randomly generated intervals and consider event $C$ as above. Note that for each interval $[s_m,\,e_m]$, we have $\PP(s_m\in \mathcal{I}^L_j, e_m\in \mathcal{I}^R_j) \geq 1/N^2$. Then, we can control
\[
\PP(C^c)\leq \sum_{j = 1}^K\prod_{m=1}^M(1-\PP(s_m\in \mathcal{I}^L_j, e_m\in \mathcal{I}^R_j))\leq K\Big(1-\frac{1}{N^2}\Big)^M \leq K\exp(-M/N^2).
\]
Therefore, choosing $M \geq N^2\log(K/\delta)$ ensures 
\[
\PP(C) \geq 1-\delta. 
\]
In the following, we work under the event $A\cap B\cap C$, which happens with probability at least $1-3\delta$.

In \emph{Step Four}, we aim to find an upper bound on the localization error. Consider the set of over-thresholding intervals 
\[
\mathcal{O}=\{m\in \mathcal{M}: \max_{s_m\leq b<e_m}W_{s_m,e_m}^{\bm{Y}}(b) > r\}.
\]
Using the event $A$ and \Cref{lemma:cusum-onechange}, we have that this set is non-empty under the weaker signal-to-noise condition \eqref{eq:snr-thm-2}. Therefore it is also non-empty under the condition \eqref{eq:snr-thm-1} since 
\[
\Delta_2 = \min_{j\in[K+1]}\sum_{i = \tau_{j-1}+1}^{\tau_j}\frac{1}{\sigma_i^2} \geq \frac{\Delta_1}{\sigma_{\max}^2}.
\]
There is a slight difference in how we select the shortest over-threshold interval, depending on whether we work under the condition \eqref{eq:snr-thm-1} or \eqref{eq:snr-thm-2}. This is related to the fact that the intervals are defined at different scales. In particular, we choose $m^*=\argmin_{m \in \mathcal{O}}(e_m-s_m+1)$ when assuming \eqref{eq:snr-thm-1} and $m^*=\argmin_{m \in \mathcal{O}}(S^w_{(s_m+1):e_m})$ when assuming \eqref{eq:snr-thm-2}. Such choices ensure that the corresponding intervals cannot contain more than one change point.

With a slight abuse of notation, we simply use $[s,e]$ to denote the interval $[s_{m^*},e_{m^*}]$. It can be shown using the same arguments as in \citet{baranowski2019narrowest} that this interval contains a single change point, which we denote as $\tau_j$, and $\min\{S^w_{s:\tau_j},S^w_{(\tau_j+1):e}\}> r^2/\kappa^2$. The estimator is $\widehat{\tau}_j = \widehat{b} = \argmax_{b \in [s,e]} W_{s,e}^{\bm{Y}}(b)$. The arguments below do not rely on the signal-to-noise conditions. To control the error between $\tau_j$ and $\widehat{b}$, note that 
\begin{align*}
(W_{s,e}^{\bm{Y}}(\tau_j))^2- (W_{s,e}^{\bm{Y}}(b))^2  &= (\langle \bm{Y},\psi_{s,e}^{\tau_j}\rangle_{w})^2 - (\langle \bm{Y},\psi_{s,e}^{b}\rangle_{w})^2 \\
&= (\langle \bm{\mu},\psi_{s,e}^{\tau_j}\rangle_{w})^2 - (\langle \bm{\mu},\psi_{s,e}^{b}\rangle_{w})^2 + (\langle \bm{\epsilon},\psi_{s,e}^{\tau_j}\rangle_{w})^2 - (\langle \bm{\epsilon},\psi_{s,e}^{b}\rangle_{w})^2 \\ &\hspace{6em}+ 2 \langle\bm{\epsilon},\psi_{s,e}^{\tau_j}\langle\psi_{s,e}^{\tau_j},\bm{\mu}\rangle_w-\psi_{s,e}^{b}\langle\psi_{s,e}^{b},\bm{\mu}\rangle_w\rangle. 
\end{align*}
Consider first $b \geq \tau_j$, and the other case can be dealt with similarly. Using \Cref{lemma:cusum-onechange} together with the identity \eqref{eq:vec-cusum}, we have 
\[
(\langle \bm{\mu},\psi_{s,e}^{\tau_j}\rangle_{w})^2 - (\langle \bm{\mu},\psi_{s,e}^{b}\rangle_{w})^2 = \frac{S^w_{(\tau_j+1):b}S^w_{s:\tau_j}}{S^w_{s:b}}\kappa^2.
\]
Using the bounds in events $A$ and $B$ \eqref{eq:prob-events}, we further have 
\begin{align*}
(\langle \bm{\epsilon},\psi_{s,e}^{\tau_j}\rangle_{w})^2 - (\langle \bm{\epsilon},\psi_{s,e}^{b}\rangle_{w})^2 &\leq r^2
\\
2 \langle\bm{\epsilon},\psi_{s,e}^{\tau_j}\langle\psi_{s,e}^{\tau_j},\bm{\mu}\rangle_w-\psi_{s,e}^{b}\langle\psi_{s,e}^{b},\bm{\mu}\rangle_w\rangle
&\leq 2r\|\psi_{s,e}^{\tau_j}\langle\psi_{s,e}^{\tau_j},\bm{\mu}\rangle_w-\psi_{s,e}^{b}\langle\psi_{s,e}^{b},\bm{\mu}\rangle_w\|_w\\
&= 2r\sqrt{\frac{S^w_{(\tau_j+1):b}S^w_{s:\tau_j}}{S^w_{s:b}}}\kappa . 
\end{align*}
Notice that if 
\[
\frac{2S^w_{(\tau_j+1):b}S^w_{s:\tau_j}}{S^w_{s:b}}\kappa^2  = \frac{2(S^w_{s:b}-S^w_{s:\tau_j})S^w_{s:\tau_j}}{S^w_{s:b}}\kappa^2 \geq \min\{S^w_{s:\tau_j},S^w_{(\tau_j+1):b}\}\kappa^2 > r^2,
\]
then $(W_{s,e}^{\bm{Y}}(\tau_j))^2- (W_{s,e}^{\bm{Y}}(b))^2 > 0$.
Since $S^w_{s:\tau_j} > r^2/\kappa^2$, we have that the estimator $\widehat{b}$ must satisfy 
\[
S^w_{(\tau_j+1):\widehat{b}} = \sum_{i = \tau_j+1}^{\widehat{b}}w_i =\sum_{i = \tau_j+1}^{\widehat{b}} \frac{1}{\sigma_i^2} \leq \frac{r^2}{\kappa^2},
\]
as, otherwise, it cannot be the maximizer of $W_{s,e}^{\bm{Y}}(b)$.
For $b < \tau_j$, following the same arguments and using the fact $S^w_{(\tau_j+1):e} > r^2/\kappa^2$, we would obtain 
\[
\sum_{\widehat{b}+1}^{\tau_j}\frac{1}{\sigma_i^2}\leq \frac{r^2}{\kappa^2}.
\]
Therefore, combining these two cases, we obtain 
\begin{equation*}
\sum_{\min{\{\tau_j,\widehat{b}}\}+1}^{\max\{\tau_j,\widehat{b}\}}\frac{1}{\sigma_i^2}\leq \frac{r^2}{\kappa^2}.
\end{equation*}
Finally, we note that we omit the details of Step Five, which deals with the recursive aspect of the algorithm and shows that the true number of change points $K$ is recovered. This step in our case follows directly by using similar adjustments to the intervals $[s_k,e_k]$ as \eqref{eq:interval-adjustment}, and incurring a slightly modified version of Lemma~3 in \citet{baranowski2019narrowest} with $\tau_j-s+1$ replaced by $S^w_{s:\tau_j}$ and $e-\tau_j$ replaced by $S^w_{(\tau_j+1):e}$.
\end{proof}

\begin{proof}[Proof of \Cref{thm:cusum}]
Note that as long as $w_i = c$ for some $c>0$, it holds that 
\begin{align*}
W_{s,e}^{\bm{v}}(b) &= \sqrt{\frac{S^w_{s:b} \,S^w_{(b+1):e}}{S^w_{s:e}}}\Bigg|\sum_{i=s}^b {\frac{w_i}{S^w_{s:b}}}v_i
- \sum_{i=b+1}^{e}{\frac{w_i}{S^w_{(b+1):e}}}v_i\Bigg| \\&= \sqrt{c}\sqrt{\frac{(b-s+1)(e-b)}{e-s+1}} \Bigg|\frac{\sum_{i = s }^b v_i}{b-s+1} - \frac{\sum_{i = b + 1}^{e} v_i}{e-b} \Bigg| =  \sqrt{c} C_{s,e}^{\bm{v}}(b).
\end{align*}
Moreover, choosing $c = 1/\sigma_{\max}^2 = \min_{i\in [N]} w_i$ ensures that the high-probability events $A$ and $B$ in the proof of \Cref{thm:cusum-hetero} still hold with the same choice of $r$. Therefore the conclusion of \Cref{thm:cusum-hetero} still holds but only with $w_i = 1/\sigma_{\max}^2$. Substituting this choice into the bound \eqref{eq:weighted-localisation}, we obtain the claimed result 
\[
\max_{j\in [K]}|\widehat{\tau}_j - \tau_j| \lesssim \frac{\sigma_{\max}^2\log(N/\delta)}{(\mu_m-\mu_h)^2}.
\]
\end{proof}

\begin{proof}[Proof of \Cref{lemma:lowerbound}] Fix a $\delta \in (0,1/2)$.
We start by considering the following two distributions $P_1$ and $P_2$ both belonging to $\mathcal{P}$, where $P_1$ is the joint distribution $N$ independent random variables such that 
\[
Y_i \sim N(\mu_m,\sigma_i^2) \quad i \leq \tau, \qquad Y_i \sim N(\mu_h,\sigma_i^2) \quad i > \tau. 
\]
$P_2$ denotes the joint distribution $N$ independent random variables such that
\[
Y_i \sim N(\mu_m,\sigma_i^2) \quad i \leq \tau+h_1, \qquad Y_i \sim N(\mu_h,\sigma_i^2) \quad i > \tau+h_1.
\]
Let $\kappa = |\mu_m-\mu_h|$ and then $\KL(P_1,P_2) = \sum_{i=\tau+1}^{\tau+h_1}\frac{\kappa^2}{\sigma_i^2}$. Now, applying \citet[][Corollary 6]{ma2024high}, a version of Le Cam's two-point lemma for high-probability lower bounds, we obtain that if $h_1$ is chosen such that 
\[
\sum_{i=\tau+1}^{\tau+h_1}\frac{\kappa^2}{\sigma_i^2} \leq \log(\frac{1}{4\delta(1-\delta)}),
\]
then $ \inf_{\widehat{\tau}} \sup_{P\in \mathcal{P}}Q(\delta,\widehat{\tau},P) \gtrsim h_1$. Therefore, we can choose it to be $\max\{h_1\in [N-\tau]:\sum_{i=\tau+1}^{\tau+h_1}\frac{\kappa^2}{\sigma_i^2} \leq \log(\frac{1}{4\delta(1-\delta)})\}$ and this value is guaranteed to be less than $N-\tau$ since $\kappa^2S^w_{(\tau+1):N} \gtrsim \log(1/\delta)$, as required in $\mathcal{P}$. The other lower bound can be obtained by considering $P_1$ and $P_3$ belonging to $\mathcal{P}$, where $P_3$ is the joint distribution of $N$ independent random variables such that 
\[
Y_i \sim N(\mu_m,\sigma_i^2) \quad i \leq \tau-h_2-1, \qquad Y_i \sim N(\mu_h,\sigma_i^2) \quad i \geq \tau-h_2. 
\]
Then we have $\KL(P_1,P_3) = \sum_{i=\tau-h_2}^\tau\frac{\kappa^2}{\sigma_i^2}$. The same arguments as before implies that if $h_2$ satisfies 
\[
\sum_{i=\tau-h_2}^{\tau}\frac{\kappa^2}{\sigma_i^2} \leq \log(\frac{1}{4\delta(1-\delta)}),
\]
then $ \inf_{\widehat{\tau}} \sup_{P\in \mathcal{P}}Q(\delta,\widehat{\tau},P) \gtrsim  h_2$. We take $\max\{h_2\in [\tau]:\sum_{i=\tau-h_2}^{\tau}\frac{\kappa^2}{\sigma_i^2} \leq \log(\frac{1}{4\delta(1-\delta)})\}$ and this value is guaranteed to be less than $\tau$ since $\kappa^2S^w_{1:\tau}\gtrsim\log(1/\delta)$, as required in $\mathcal{P}$. Combining these two cases yields our claim. 
\end{proof}

\section{Extensions to dependent scores and unknown score variances}\label{sec:extension}

In this section, we generalise the theoretical guarantees regarding WCP. 
Recall the setup in \Cref{sec:theory}, where there are $K$ change points $\{\tau_j\}_{j\in [K]}$ in a paragraph $\bm{X}=(X_1,\dotsc,X_N)$. Let $\mu_h = \mathbb{E}[\phi(X_h)]$ and $\mu_m = \mathbb{E}[\phi(X_m)]$ denote the expected scores for a human-authored sentence $X_h$ and an LLM-authored sentence $X_m$, respectively.

\begin{assumption}\label{asp:extension}
We let $Y_i = \phi(X_i) = \mu_i+\epsilon_i, i\in [N]$ denote the scores on the segment $X_i$. Suppose there are $K$ change points $\{\tau_j\}_{j\in [K]}$ in a paragraph $\bm{X}=(X_1,\dotsc,X_N)$. Assume the following:
\begin{enumerate}[leftmargin=*]
\item (\textbf{Unknown $
\sigma_i^2$}) Suppose that $\sigma_i^2$ are unknown but we have access to estimators $\hat{\sigma}_i^2$ that are independent of $\bm{X}$ such that with probability at least $1-\delta$
\begin{equation}\label{eq:unknown-weights}
\mathcal{E} = \{\max_{i = 1,\dotsc,N}|\hat{\sigma}^2_i/\sigma_i^2 - 1|\leq R\},
\end{equation}
this event happens, for some $R \in (0,1)$.   
\item (\textbf{Dependence}) The random variables $Z_i = \epsilon_i/\sigma_i$ has mean zero and $|Z_i| \leq B \asymp 1$. For $1\leq i < j \leq N$, consider the total-variation dependence measure as in \citep[e.g.][]{samson2000concentration,kontorovich2008concentration}. Let 
\[
{\eta}_{ij}
=
\sup_{z_{1:i-1},\,z,z'}
\left\|
\mathcal{L}\!\left(
Z_{j:N}\mid Z_{1:i-1}=z_{1:i-1},\, Z_i=z
\right)
-
\mathcal{L}\!\left(
Z_{j:N}\mid Z_{1:i-1}=z_{1:i-1},\, Z_i=z'
\right)
\right\|_{\mathrm{TV}},
\]
where $\mathcal{L}(Z\mid Y=y)$ denotes the conditional distribution of $Z$ given $Y=y$. Assume that ${\eta}_{ij} \leq \rho(j-i)$ for some $\rho$ and 
\[
\Lambda_{\rho}
=
1+\sum_{k=1}^{\infty}\sqrt{\rho(k)}
<\infty. 
\]   
\end{enumerate}
\end{assumption}

\begin{proof}[Proof of \Cref{thm:extension}]
Throughout the proof, we shall work under event $\mathcal{E}$. 
The main modification compared to the proof of \Cref{thm:cusum-hetero} is in the concentration inequality step.  Recall the two events that we need to control are
\begin{equation}\label{eq:prob-events-2}
\begin{aligned}
A &= \{\max_{s,b,e:1\leq s\leq b< e\leq N}|W_{s,e}^{\bm{\epsilon}}(b)|\leq r\},\\
B &= \Big\{\frac{|\langle \bm{\mu}_{\psi},\bm{\epsilon}\rangle_w|}{\|\bm{\mu}_{\psi}\|_w}\leq r,\ \forall (j,s,e,b)\in\mathcal{D}\Big\}.
\end{aligned}
\end{equation}
We essentially apply the concentration inequality \cite[][eq.\ (2.23)]{samson2000concentration} that allows dependence across random variables. Let $ V_i=\frac{Z_i}{B}$ and $b_R=\frac{1}{1-R}$ so that under event $\mathcal{E}$, $ w_i\le b_R\sigma_i^{-2}.$

Let \(\Gamma_N\) be the upper triangular matrix with $1$ on the diagonal and $\sqrt{\eta_{ij}}$ for $i <j$, and note that 
\[
\|\Gamma_N\|_{\mathrm{op}} \leq 1+\sum_{k=1}^\infty\sqrt{\rho(k)} \leq \Lambda_{\rho}.
\]

The concentration inequality \cite[][eq.\ (2.23)]{samson2000concentration} implies for weight vectors \(b_i\),
\[
\mathbb P\left(
\left|
\sum_{i=1}^N V_i b_i
-
\mathbb E\sum_{i=1}^N V_i b_i
\right|
\ge t
\right)
\le
\exp\left\{
-\frac{t^2}{8\sigma_*^2(b)\Lambda_{\rho}^2}
\right\},
\]
where
\[
\sigma_*^2(b)
\leq
\sum_{i=1}^N b_i^2 .
\]

For event \(A\), fix \((s,b,e)\), we have
\[
W_{s,e}^{\epsilon}(b)
=
\left|
\sum_{i=1}^N
w_i\sigma_i
\psi^b_{s,e}(i) Z_i
\right| =  \left|
\sum_{i=1}^N
V_i b_i^A
\right|,
\]
where $ b_i^A
=
B w_i\sigma_i\psi^b_{s,e}(i)$, and $\mathbb{E}V_i = 0$ if the variance estimators $\hat{\sigma}_i^2$ are independent of the noise.  Then, it holds that
\[
\begin{aligned}
\sum_{i=1}^N
\left(B w_i\sigma_i
\psi^b_{s,e}(i)\right)^2 &=
B^2
\sum_{i=1}^N
(w_i\sigma_i^2)
w_i
\{\psi^b_{s,e}(i)\}^2 \le
B^2 b_R
\sum_{i=1}^N
w_i
\{\psi^b_{s,e}(i)\}^2
=
B^2b_R ,
\end{aligned}
\]
where we use
\[
w_i\sigma_i^2\le b_R,
\qquad
\sum_{i=1}^N
w_i\{\psi^b_{s,e}(i)\}^2=1.
\]

Similarly, for event \(B\), we have for any $u$
\[
\frac{
|\langle u,\epsilon\rangle_{w}|
}{
\|u\|_{w}
}
=
\left|
\sum_{i=1}^N
\frac{
w_i\sigma_i u_i
}{
\left(\sum_{\ell=1}^N w_\ell u_\ell^2\right)^{1/2}
}
Z_i
\right| = \left|
\sum_{i=1}^N
V_i b_i^{B}
\right|,
\]
where $ b_i^{B}
=
B
\frac{
w_i\sigma_i u_i
}{
\left(\sum_{\ell=1}^N w_\ell u_\ell^2\right)^{1/2}
},$ and 
\[
\begin{aligned}
\sum_{i=1}^N (b_i^{B})^2
=
B^2
\frac{
\sum_{i=1}^N
w_i^2\sigma_i^2u_i^2
}{
\sum_{\ell=1}^N w_\ell u_\ell^2
}  
=
B^2
\frac{
\sum_{i=1}^N
(w_i\sigma_i^2)w_i u_i^2
}{
\sum_{\ell=1}^N w_\ell u_\ell^2
}  
\le
B^2b_R
\frac{
\sum_{i=1}^N w_i u_i^2
}{
\sum_{\ell=1}^N w_\ell u_\ell^2
}
=
B^2b_R .
\end{aligned}
\]

Therefore, we can choose 
\[
r
\asymp
B \Lambda_{\rho}
\sqrt{
\frac{\log(N/\delta)}{1-R}
},         
\]
to ensure $\mathbb{P}(A^c)$ and $\mathbb{P}(B^c)$ are less than $\delta$.

For Step 3, since we work under the event $\mathcal{E}$, we have $a_RS^{\sigma}_{u:v}\leq S^w_{u:v} \leq b_R S^{\sigma}_{u:v}$, where $a_R = 1/(1+R)$, $b_R = 1/(1-R)$ and $S^{\sigma}_{u:v} = \sum_{i=u}^v\sigma_i^{-2}$. Therefore the signal-to-noise condition in \eqref{eq:snr-2} implies 
\[
(\mu_m-\mu_h)^2 \hat{\Delta}_2 \gtrsim r,
\]
with $\hat{\Delta}_2 = \min_{j\in[K+1]}S^w_{(\tau_{j-1}+1):\tau_j}$. Once this condition is satisfied, following the same arguments as in the proof of \Cref{thm:cusum-hetero}, we would end with 
\begin{equation*}
\sum_{\min{\{\tau_j,\widehat{b}}\}+1}^{\max\{\tau_j,\widehat{b}\}}\frac{1}{\hat{\sigma}_i^2}\leq \frac{r^2}{\kappa^2}.
\end{equation*}
Then, again incurring the property in the event $\mathcal{E}$, we have 
\[
\sum_{\min{\{\tau_j,\widehat{b}}\}+1}^{\max\{\tau_j,\widehat{b}\}}\frac{1}{{\sigma}_i^2}\leq \frac{r^2}{a_R\kappa^2} \asymp \frac{1+R}{1-R}B^2\Lambda_\rho^2\frac{\log(N/\delta)}{\kappa^2}.
\]
\end{proof}

\section{Generalized change point detection}\label{sec:textCP}
Another natural approach to enhance the performance of the vanilla algorithm is to replace the average of scores $\phi(X_t)$ in CUSUM with segment-level scores $\phi(X_{t_1:t_2})$, obtained by applying, for example, the FastDetectGPT statistic to the concatenated text $X_{t_1:t_2}$. This yields the following generalized CUSUM statistic: 
\begin{equation}\label{eq:segmentwise}
G_{s,e}^{\bm{X}}(t) = \sqrt{\frac{S^w_{s:t} \, S^w_{(t+1):e}}{S^w_{s:e}}} \big|\phi(X_{s:t}) - \phi(X_{(t+1):e})\big|.
\end{equation}
Substituting $A_{s,e}^{\bm{X}}(b)$ in \Cref{alg:not_algorithm} with our newly defined $G_{s,e}^{\bm{X}}(b)$, and $B(s,e)$ with $S^w_{s:e}$, yields the resulting algorithm, which we refer to as \textbf{GCP}, short for \underline{G}eneralized \underline{C}hange \underline{P}oint detection.

Compared to the standard CUSUM statistic in \eqref{eq:CUSUM}, \eqref{eq:segmentwise} offers two advantages: (i) It aggregates information across multiple sentences, enabling the LLM detector to achieve higher classification accuracy than detecting each sentence individually. (ii) It implicitly incorporates sentence-specific attributes, such as sentence length, into the resulting score, since longer sentences likely exert a larger influence on the segment-level score. While being more general than the weighted version in \eqref{eq:weighted_CUSUM}, we show later in \Cref{prop:equivalence} that \eqref{eq:segmentwise} can be equivalent to \eqref{eq:weighted_CUSUM} under specific choices of $\phi$.

However, one limitation of GCP is its computational cost. In particular, implementing \Cref{alg:not_algorithm} with the generalized CUSUM statistic requires computing the score function $\phi$ on various segments of $\bm{X}$, due to the recursive nature of the algorithm. In contrast, VCP and WCP only require a single pass of $\phi$ over each individual sentence $X_i$. 

In the next section, we establish the equivalence between WCP and GCP under appropriate conditions, so that the minimax optimality of WCP applies to GCP as well.

\subsection{Equivalence of GCP and WCP}\label{sec:equivalence}
In this section, we show the equivalence between \eqref{eq:segmentwise} and \eqref{eq:weighted_CUSUM} when $\phi$ takes certain general forms that are used in the literature \citep[e.g.][]{mitchell2023detectgpt,bao2024fastdetectgpt,zhou2025adadetectgpt}. Let $Z_t= \mathcal{R}(X_t)$ denote the text rewritten by some LLM given input text $X_t$,  $X_{t,i}$ denote the $i$-th token in the $t$-th sentence, and $X_{t,<i}$ denote all the tokens before $i$-th token in the $t$-th sentence. Assume $Z_t$ and $X_t$ share the same number of tokens $n_t$, which can be achieved by zero-padding the shorter sequence. Consider the detector $\phi$ that takes the form of either a zero-shot detection statistic
\[
\phi_1(X_t) = \frac{1}{n_t}\log p_{\widehat{\varphi}}(X_t),
\]
where $p_{\widehat{\varphi}}$ is some possible source model to be detected, or a ML-based detection statistic
\[
\phi_2(X_t) = \frac{1}{n_t}\log\frac{p_{{\widehat\varphi}}(X_t)}{p_{\widehat{\varphi}}(Z_t)} =\frac{1}{n_t}\sum_{i=1}^{n_t}\log\frac{p_{\widehat{\varphi}}(X_{t,i}|X_{t,<i})}{p_{\widehat{\varphi}}(Z_{t,i}|Z_{t,<i})},
\]
where $p_{\widehat{\varphi}}$ is a classifier trained to maximally distinguish $X_t$ and $Z_t$.
\begin{proposition}\label{prop:equivalence}
Under the above choices of detection statistics $\phi$, if for any input text $\bm{X}$ with $N$ sentences, where each sentence $X_i$ has $n_i$ tokens, it holds that 
\begin{equation}\label{eq:equi-condition}
p_{\widehat{\varphi}}(X_{i,j}|X_{i,<j}) = p_{\widehat{\varphi}}(X_{i,j}|X_{<i,<j}), 
\end{equation}
for $i \in [N], j\in [n_i]$, 
then 
\[
G_{s,e}^{\bm{X}}(b) = W_{s,e}^{\bm{Y}}(b), \quad \forall b\in[N],
\]
with $Y = (\phi(X_1),\dotsc,\phi(X_N))$ and $w_i = n_i$.
\end{proposition}
\begin{proof}[Proof of \Cref{prop:equivalence}] We directly prove the case for $\phi_2$, and the case for $\phi_1$ is a special case by setting $p_{\widehat{\varphi}}(Z_t) = 1$. Under $\phi_2$, we have 
\begin{align*}
G_{s,e}^{\bm{X}}(b)= \sqrt{\frac{S^w_{s:b} \,S^w_{(b+1):e}}{S^w_{s:e}}}  &\Bigg|\frac{1}{S^w_{s:b}}\sum_{i=s}^b\sum_{j=1}^{n_i}\log\frac{p_{\widehat{\varphi}}(X_{i,j}|X_{<i,<j})}{p_{\widehat{\varphi}}(Z_{i,j}|Z_{<i,<j})} - \frac{1}{S^w_{(b+1):e}}\sum_{i=b+1}^e\sum_{j=1}^{n_i}\log\frac{p_{\widehat{\varphi}}(X_{i,j}|X_{<i,<j})}{p_{\widehat{\varphi}}(Z_{i,j}|Z_{<i,<j})}\Bigg| \\
W_{s,e}^{\bm{Y}}(b) = \sqrt{\frac{S^w_{s:b} \,S^w_{(b+1):e}}{S^w_{s:e}}}  &\Bigg|\frac{1}{S^w_{s:b}}\sum_{i=s}^b\sum_{j=1}^{n_i}\log\frac{p_{\widehat{\varphi}}(X_{i,j}|X_{i,<j})}{p_{\widehat{\varphi}}(Z_{i,j}|Z_{i,<j})} - \frac{1}{S^w_{(b+1):e}}\sum_{i=b+1}^e\sum_{j=1}^{n_i}\log\frac{p_{\widehat{\varphi}}(X_{i,j}|X_{i,<j})}{p_{\widehat{\varphi}}(Z_{i,j}|Z_{i,<j})}\Bigg|
\end{align*}
Therefore, under the assumption that for any $i \in [N]$ and $j\in [n_i]$, 
\[
p_{\widehat{\varphi}}(X_{i,j}|X_{i,<j}) = p_{\widehat{\varphi}}(X_{i,j}|X_{<i,<j}),
\]
where $X$ is the input text, 
it holds that $W_{s,e}^{\bm{Y}}(b) = D_{s,e}^{\bm{X}}(b)$ for all $b \in [N]$.
\end{proof}

\section{Additional numerical results}\label{sec:addition-numerical}
\setcounter{table}{0}
\renewcommand{\thetable}{A\arabic{table}}

\subsection{Table on real-data analysis}
\begin{table}[H]
\centering
\vspace{-12pt}
\caption{Results on the real-world CoAuthor dataset.}
\label{tab:real-data-coauthor}
\setlength{\tabcolsep}{3pt}
\small
\begin{tabular}{lcc}
\toprule
& WD & CE \\
\midrule
TextTiling & 0.66 & -6.76 \\
LLMPred & 0.64 & -6.25 \\
SenPred & 0.70 & -11.68 \\
Voting & 0.51 & -2.65 \\
SegFormer & 0.39 & \textbf{2.45} \\
VCP & \textbf{0.36} & 2.63 \\
WCP & \textbf{0.36} & 2.71 \\
\bottomrule
\end{tabular}
\end{table}

\subsection{Results on multiple change points}\label{sec:simu-multi-cp}
In this part, we consider settings with multiple change points where the number of change points $K$ varies within $\{1,2,3,5,8\}$. For each document, we evenly split the text into $K+1$ segments, each containing approximately the same number of sentences. The 1st, 3rd, 5th, $\cdots$ segments are unchanged, while the remaining segments are written by an LLM. PaLD is not included in this or subsequent experiments due to its computational cost. 

\begin{table}[htbp]
\centering
\caption{Results on multiple change-point detection on the Story dataset. The best results are presented in \textbf{bold}.}
\label{tab:multiple_cp}
\small
\setlength{\tabcolsep}{3pt}
\renewcommand{\arraystretch}{1.1}
\resizebox{\textwidth}{!}{
\begin{tabular}{llcccccccccc}
\toprule
\multirow{2}{*}{Model} & \multirow{2}{*}{Method}
& \multicolumn{2}{c}{$K=1$}
& \multicolumn{2}{c}{$K=2$}
& \multicolumn{2}{c}{$K=3$}
& \multicolumn{2}{c}{$K=5$}
& \multicolumn{2}{c}{$K=8$} \\
\cmidrule(lr){3-4}\cmidrule(lr){5-6}\cmidrule(lr){7-8}\cmidrule(lr){9-10}\cmidrule(lr){11-12}
& & WD & CE & WD & CE & WD & CE & WD & CE & WD & CE \\
\midrule
\multirow{6}{*}{Claude 4.5}
& TextTiling & 0.75 & -2.68 & 0.77 & -3.80 & 0.72 & -4.47 & 0.64 & -4.64 & 0.53 & -3.06 \\
& LLMPred   & 1.42 & -4.24 & 1.23 & -4.68 & 1.05 & -4.43 & 0.52 & 2.27  & 0.49 & 4.96 \\
& SenPred   & 2.18 & -8.70 & 2.24 & -13.80 & 2.19 & -17.91 & 1.89 & -23.38 & 1.31 & -24.68 \\
& Voting    & 0.38 & -1.23 & 0.45 & -2.19 & 0.46 & -2.78 & 0.38 & -2.71 & \textbf{0.33} & -2.56 \\
& VCP       & 0.31 & -1.73 & 0.33 & -2.81 & \textbf{0.31} & -2.90 & \textbf{0.32} & {-3.35} & \textbf{0.33} & -1.22 \\
& WCP       & \textbf{0.26} & \textbf{-1.11} & \textbf{0.29} & \textbf{-1.01} & \textbf{0.31} & \textbf{-1.13} & 0.33 & \textbf{-0.91} & \textbf{0.33} & \textbf{0.68} \\
\midrule
\multirow{6}{*}{GPT-5-mini}
& TextTiling & 0.96 & -3.47 & 1.05 & -5.34 & 0.91 & -6.10 & 0.75 & -6.50 & 0.61 & -5.20 \\
& LLMPred   & 1.34 & -4.02 & 1.20 & -4.74 & 1.12 & -5.06 & 0.56 & 1.45  & 0.51 & 4.46 \\
& SenPred   & 2.68 & -10.94 & 2.89 & -17.90 & 2.92 & -24.38 & 2.42 & -30.43 & 1.79 & -33.54 \\
& Voting    & 0.40 & -0.50 & 0.47 & -0.28 & 0.46 & 0.86 & 0.82 & -7.60 & 0.65 & -6.39 \\
& VCP       & 0.41 & \textbf{-0.47} & 0.50 & -1.64 & 0.51 & -2.01 & 0.51 & -1.52 & 0.46 & \textbf{0.65} \\
& WCP       & \textbf{0.39} & {-0.88} & \textbf{0.41} & \textbf{0.05} & \textbf{0.43} & \textbf{0.24} & \textbf{0.44} & \textbf{0.76} & \textbf{0.43} & 1.91 \\
\bottomrule
\end{tabular}
}
\end{table}

Table~\ref{tab:multiple_cp} shows that WCP remains competitive as the number of change points increases. Across both LLM generators, WCP achieves the best or near-best WD in most settings. VCP also performs strongly because it uses the same change-point structure, but its larger WD in many settings indicates that unweighted segmentation is less powerful when sentence-level detection scores are heterogeneous. In contrast, the baseline methods often incur larger WD or CE, which is consistent with the single-change-point results.

\subsection{Sensitivity analysis}\label{sec:simu-sensitive}
We further study two data-generation factors that may affect localization performance: adversarial attacks on the text and the proportion of LLM-written content.

\paragraph{Adversarial attack.} We consider two adversarial settings: (1) decoherence and (2) paraphrasing. For decoherence, we perturb LLM-generated sentences to reduce their coherence, making them closer to human-written text that may not be perfectly coherent. Specifically, we use a lightweight attack that randomly swaps one adjacent word pair within each LLM-generated sentence. For paraphrasing, we perturb human-written sentences by using a language model to rephrase them, making them closer to LLM-generated text. Results under the two attacks are presented in Table~\ref{tab:adversarial-attack}. 

Although these perturbations increase the difficulty of localization, WCP generally achieves the best performance across the two adversarial settings. Its WD also increases more mildly than those of the baselines, suggesting that the change-point formulation remains useful when sentence-level detection signals are weakened.

\begin{table}[htbp]
\centering
\caption{Results on two adversarial attack setups.}
\label{tab:adversarial-attack}
\small
\setlength{\tabcolsep}{5pt}
\renewcommand{\arraystretch}{1.1}
\begin{tabular}{llcccccc}
\toprule
\multirow{2}{*}{Setup} & \multirow{2}{*}{Method}
& \multicolumn{2}{c}{WikiQA}
& \multicolumn{2}{c}{News}
& \multicolumn{2}{c}{Story} \\
\cmidrule(lr){3-4}\cmidrule(lr){5-6}\cmidrule(lr){7-8}
& & WD & CE & WD & CE & WD & CE \\
\midrule
\multirow{5}{*}{decoherence}
& TextTiling & 0.418 & -0.465 & 1.320 & -5.110 & 0.969 & -3.480 \\
& LLMPred & 0.490 & -0.687 & 1.193 & -3.740 & 1.217 & -3.450 \\
& SenPred & 0.696 & -2.172 & 2.941 & -12.150 & 2.892 & -11.800 \\
& Voting & 0.441 & -0.657 & 0.950 & -3.480 & 0.929 & -3.390 \\
& WCP & \textbf{0.370} & \textbf{-0.010} & \textbf{0.380} & \textbf{0.000} & \textbf{0.354} & \textbf{0.000} \\
\midrule
\multirow{5}{*}{paraphrasing}
& TextTiling & 0.417 & -0.485 & 1.313 & -5.130 & 0.957 & -3.430 \\
& LLMPred & 0.540 & -0.949 & 0.879 & -2.230 & 1.216 & -3.500 \\
& SenPred & 0.675 & -2.162 & 2.858 & -11.800 & 3.076 & -12.580 \\
& Voting & \textbf{0.400} & -0.455 & 0.944 & -3.450 & 1.018 & -3.910 \\
& WCP & 0.412 & \textbf{-0.010} & \textbf{0.421} & \textbf{0.000} & \textbf{0.379} & \textbf{0.000} \\
\bottomrule
\end{tabular}
\end{table}

\paragraph{Varying proportion of LLM-written content.}

We next use the Story dataset to study how performance changes with the proportion of LLM-written content. For each document, we keep an initial proportion of sentences as human-written and generate the remaining sentences using an LLM. We vary this proportion within $\{5\%, 10\%, 20\%, 40\%, 80\%\}$. Table~\ref{tab:simu-prop} shows the results of WCP and baselines. WCP consistently achieves smaller WD and more stable CE across different LLM-authored proportions, indicating that the proposed localization procedure remains effective when the relative lengths of human-written and LLM-written segments vary.

\begin{table}[htbp]
\centering
\caption{Performance of various methods on different proportions of LLM-written content on the Story dataset.}
\label{tab:simu-prop}
\small
\setlength{\tabcolsep}{3pt}
\renewcommand{\arraystretch}{1.1}
\begin{tabular}{lcc|cc|cc|cc|cc}
\toprule
& \multicolumn{2}{c}{5\%} 
& \multicolumn{2}{c}{10\%} 
& \multicolumn{2}{c}{20\%}
& \multicolumn{2}{c}{40\%}
& \multicolumn{2}{c}{80\%} \\
\cmidrule(lr){2-3} \cmidrule(lr){4-5} \cmidrule(lr){6-7} \cmidrule(lr){8-9} \cmidrule(lr){10-11}
Method & WD & CE 
& WD & CE 
& WD & CE
& WD & CE
& WD & CE \\
\midrule
TextTiling & 0.97 & -7.77 & 0.97 & -7.77 & 0.94 & -7.62 & 0.93 & -6.98 & 0.92 & -5.58 \\
LLMPred & 0.93 & -8.72 & 0.94 & -8.59 & 0.90 & -8.04 & 0.86 & -7.89 & 0.91 & -8.20 \\
SenPred    & 1.00 & -37.44 & 1.00 & -36.80 & 0.99 & -35.73 & 0.99 & -36.46 & 1.00 & -37.16 \\
Voting     & 0.69 & -8.90 & 0.48 & -5.40 & 0.46 & -5.12 & 0.52 & -6.58 & 0.81 & -11.19 \\
WCP        & \textbf{0.26} & \textbf{-2.84} & \textbf{0.21} & \textbf{-2.99} & \textbf{0.24} & \textbf{-2.80} & \textbf{0.29} & \textbf{-2.65} & \textbf{0.44} & \textbf{-2.71} \\
\bottomrule
\end{tabular}
\end{table}

\paragraph{Varying detectors for implementing algorithms.} Our method requires a sentence-level detector as input. While the main text uses AdaDetectGPT, here we examine whether the proposed change-point procedure remains effective with alternative detector scores, including log-likelihood \citep[LL;][]{gehrmann2019gltr}, log-likelihood log-rank ratio \citep[LRR;][]{su2023detectllm}, and Fast-DetectGPT \citep[FDGPT;][]{bao2024fastdetectgpt}. For a concise comparison, we report SenPred and Voting, which are among the strongest baselines in the main experiments. Table~\ref{tab:detector} shows that WCP remains competitive across all three detector choices, suggesting that our gain over baselines is not tied to a particular detector.

\begin{table}[htbp]
\centering
\caption{Robustness to different sentence-level detector scores.}
\label{tab:detector}
\small
\setlength{\tabcolsep}{5pt}
\begin{tabular}{llcccccc}
\toprule
\multirow{2}{*}{Detector} & \multirow{2}{*}{Method} & \multicolumn{2}{c}{WikiQA} & \multicolumn{2}{c}{News} & \multicolumn{2}{c}{Story} \\
\cmidrule(lr){3-4} \cmidrule(lr){5-6} \cmidrule(lr){7-8}
& & WD & CE 
& WD & CE 
& WD & CE \\
\midrule
\multirow{3}{*}{LL} 
& SenPred & 0.57 & -1.34 & 2.85 & -11.09 & 2.38 & -9.08 \\
& Voting  & 0.40 & -0.07 & 0.77 & -2.46  & 0.88 & -2.99 \\
& WCP     & \textbf{0.30} & \textbf{-0.05}  & \textbf{0.28} & \textbf{-1.18}   & \textbf{0.37} & \textbf{-1.94} \\
\midrule
\multirow{3}{*}{LRR} 
& SenPred & 0.61 & -1.57 & 3.11 & -12.31 & 2.93 & -11.98 \\
& Voting  & 0.43 & -0.17 & 0.96 & -3.26  & 0.92 & -3.31 \\
& WCP     & \textbf{0.30} & \textbf{-0.02}  & \textbf{0.37} & \textbf{-1.43}   & \textbf{0.36} & \textbf{-1.70} \\
\midrule
\multirow{3}{*}{FDGPT} 
& SenPred & 0.47 & -1.26 & 2.87 & -11.19 & 2.98 & -11.79 \\
& Voting  & 0.39 & -0.10 & 0.59 & -1.92  & 0.85 & -3.04 \\
& WCP     & \textbf{0.30} & \textbf{-0.03}  & \textbf{0.29} & \textbf{-1.69}   & \textbf{0.31} & \textbf{-1.53} \\
\bottomrule
\end{tabular}
\end{table}

\textbf{Varying tuning parameters.} We conducted additional sensitivity analyses for the threshold parameter $r$ and the weight exponent, which are reported in Table~\ref{tab:vary_r} and Table~\ref{tab:vary_power}. We omit a separate sensitivity table for $M$ because NOT-type algorithms are empirically stable once $M$ is sufficiently large. There additional sensitivity analyses are conducted for WCP on the Story dataset, which contains a single change point. Table~\ref{tab:vary_r} shows that WD is relatively stable under different choices of $r$ and the performance remains competitive compared to the results reported in Table \ref{tab:single-cp}. CE can be affected by the value of $r$, with a lower value of $r$ leading to over-estimation of the number of change points. For the weight exponent, Table~\ref{tab:vary_power} shows that setting the exponent to $2$ gives the best empirical performance in our experiments.

\begin{table}[H]
\centering
\caption{Sensitivity to $r$.}
\label{tab:vary_r}
\begin{tabular}{lcccc}
\toprule
\multirow{2}{*}{\textbf{Method}} & \multicolumn{2}{c}{\textbf{Claude 4.5}} & \multicolumn{2}{c}{\textbf{GPT-5-mini}} \\
\cmidrule(lr){2-3} \cmidrule(lr){4-5}
& \textbf{WD} & \textbf{CE} & \textbf{WD} & \textbf{CE} \\
\midrule
$r=0.25$ & 0.34{\scriptsize$\pm$0.03} & -5.08{\scriptsize$\pm$0.27} & 0.54{\scriptsize$\pm$0.03} & -5.40{\scriptsize$\pm$0.30}\\
$r=0.5$ & 0.30{\scriptsize$\pm$0.03} & -2.94{\scriptsize$\pm$0.23} & 0.48{\scriptsize$\pm$0.02} & -3.01{\scriptsize$\pm$0.25} \\
$r=1$ & \textbf{0.25{\scriptsize$\pm$0.02}} & -0.80{\scriptsize$\pm$0.18} & 0.37{\scriptsize$\pm$0.02} & \textbf{-0.47{\scriptsize$\pm$0.17}} \\
$r=2$ & 0.26{\scriptsize$\pm$0.02} & \textbf{0.50{\scriptsize$\pm$0.08}} & 0.30{\scriptsize$\pm$0.01} & 0.65{\scriptsize$\pm$0.08} \\
$r=4$ & 0.28{\scriptsize$\pm$0.01} & 0.92{\scriptsize$\pm$0.03} & \textbf{0.29{\scriptsize$\pm$0.00}} & 0.99{\scriptsize$\pm$0.01} \\
\bottomrule
\end{tabular}
\end{table}

\begin{table}[H]
\centering
\caption{Sensitivity to weight exponent.}
\label{tab:vary_power}
\begin{tabular}{lcccc}
\toprule
\multirow{2}{*}{\textbf{Method}} & \multicolumn{2}{c}{\textbf{Claude 4.5}} & \multicolumn{2}{c}{\textbf{GPT-5-mini}} \\
\cmidrule(lr){2-3} \cmidrule(lr){4-5}
& \textbf{WD} & \textbf{CE} & \textbf{WD} & \textbf{CE} \\
\midrule
$p=0$ & 0.36{\scriptsize$\pm$0.02} & -1.88{\scriptsize$\pm$0.17} & 0.45{\scriptsize$\pm$0.02} & -1.78{\scriptsize$\pm$0.18} \\
$p=0.25$ & 0.34{\scriptsize$\pm$0.02} & -1.70{\scriptsize$\pm$0.16} & 0.44{\scriptsize$\pm$0.02} & -1.61{\scriptsize$\pm$0.18}\\
$p=0.5$ & 0.33{\scriptsize$\pm$0.02} & -1.51{\scriptsize$\pm$0.17} & 0.43{\scriptsize$\pm$0.02} & -1.30{\scriptsize$\pm$0.16} \\
$p=1$ & 0.32{\scriptsize$\pm$0.02} & -1.07{\scriptsize$\pm$0.17} & 0.41{\scriptsize$\pm$0.02} & -0.99{\scriptsize$\pm$0.15} \\
$p=2$ & \textbf{0.25{\scriptsize$\pm$0.02}} & \textbf{-0.80{\scriptsize$\pm$0.18}} & \textbf{0.37{\scriptsize$\pm$0.02}} & \textbf{-0.47{\scriptsize$\pm$0.17}} \\
\bottomrule
\end{tabular}
\end{table}

\subsection{Additional experiments on token-level detection}\label{sec:token-experiments}
We further evaluate the proposed localization idea in a token-level setting, which provides a finer-grained test of mixed human--LLM editing. We consider cases where either 40 or 100 tokens are generated by an LLM. We compare our method with Voting and TokenPred, where TokenPred applies the detection statistic to a 20-token window around each target token before classification. The results in Table~\ref{tab:llm-generated-tokens} show that our method achieves substantially smaller WD and more stable CE, indicating that the change-point formulation remains useful.

\begin{table}[htbp]
\centering
\caption{Token-level localization under different lengths of LLM-generated content.}
\label{tab:llm-generated-tokens}
\small
\setlength{\tabcolsep}{5pt}
\begin{tabular}{llcc}
\toprule
Number & Methods & WD & CE \\
\midrule
\multirow{3}{*}{40}
& TokenPred & 10.62 & -42.58 \\
& Voting & 1.127 & -4.58 \\
& WCP & \textbf{0.379} & \textbf{-0.98} \\
\midrule
\multirow{3}{*}{100}
& TokenPred & 21.483 & -90.818 \\
& Voting & 2.611 & -9.737 \\
& WCP & \textbf{0.222} & \textbf{0.172} \\
\bottomrule
\end{tabular}
\end{table}

\section{Experiment details}\label{sec:exp-details}

\textbf{Implementation details}. For VCP and WCP in the reported experiments, we fix the input parameter $M$ in Algorithm~\ref{alg:not_algorithm} at 200 and set $r=\sqrt{\log(N)}$. 

For LLMPred, we use the following prompt to identify LLM-generated segments:

\begin{tcolorbox}[colback=white,colframe=black!50, title=Prompt for implementing LLMPred]
\texttt{You are an expert in determining the LLM-written sentences in text. Return ONLY the sentence indices (starting from 0) that are written by a language model. \\\\ The text to detect is: [filled in the bracket]. The sentence indices are: \\}
\end{tcolorbox}

For TextTiling, we use the implementation in \texttt{nltk.tokenize.texttiling} \citep{bird2009natural}. 

For SegFormer in the CoAuthor experiment, we use the implementation provided by \citet{zeng2024aisentence}, which is publicly available on GitHub\footnote{\url{https://github.com/douglashiwo/AISentenceDetection/blob/main/A-Segment_Detection_Models/SegFormer.zip}}.

\textbf{Human--LLM co-authored text generation procedure}. As mentioned in Section~\ref{sec:exp}, we asked one LLM to continue pieces of human-written text to obtain semantically fluent co-authored text. The prompt for generating the LLM-written text is:

\begin{tcolorbox}[colback=white,colframe=black!50, title=Prompt on generating the LLM-written text]
\texttt{[System prompt] You are a creative writing assistant. Continue the given text naturally and fluently, matching the style and tone of what came before. \\\\ User prompt: Continue the following text by writing approximately \{k\} more sentences. Return ONLY the new continuation sentences, without repeating the original text. \\\\ \{preceding text filled in the bracket\} }
\end{tcolorbox}

For generating the texts for the paragraph-level detection, we makes some slight modification the prompt according:
\begin{tcolorbox}[colback=white,colframe=black!50, title=Prompt on generating the LLM-written text]
\texttt{[System prompt] Continue the supplied text naturally. Return only the new continuation, without repeating the supplied text or adding explanations, headings, or commentary.\\\\ Continue the following text by writing exactly \{k\} new sentences. Return the new sentences as one paragraph.}\\\\ \{preceding text filled in the bracket\} 
\end{tcolorbox}

Without loss of generality, we perform sampling with the temperature set to 0.8. Claude 4.5 in the main text corresponds to \texttt{claude-haiku-4-5}, and the GPT-5-mini model corresponds to \texttt{gpt-5-mini}. 

\textbf{Settings for LLM detectors.} We consider several LLM detectors: AdaDetectGPT, LL, LRR, and FastDetectGPT. For AdaDetectGPT, its scoring model is a pretrained transformation applied consistently across all experiments. For FastDetectGPT, we set \(\texttt{score}\) and \(\texttt{sample}\) to be the same model, chosen from \texttt{google/gemma-2-9b-it}. Finally, LL and LRR also require a scoring model, which we set to \texttt{google/gemma-2-9b-it}, as in FastDetectGPT. 

\textbf{Evaluation details}. 
Let $C_i^{(r)}$ and $C_i^{(h)}$ denote the numbers of true and estimated change points in the $i$-th sliding window.
The WindowDiff metric is defined as:
\[
\mathrm{WD} = \frac{1}{T-k}\sum_{i=1}^{T-k} |C_i^{(r)}-C_i^{(h)}|.
\]

\textbf{The collection of Coauthor dataset}. The CoAuthor data were collected through a custom text-editor interface. At the beginning of each session, writers were given a prompt and asked to write either a creative story or an argumentative essay. During writing, they could request GPT-3 suggestions through a shortcut whenever needed. The interface logged the full interaction trace, including writers' own text, GPT-3 suggestions, whether suggestions were accepted or dismissed, and subsequent edits to accepted suggestions or previous text. In total, CoAuthor contains 1,445 writing sessions produced by 63 writers interacting with four GPT-3 instances, including 830 creative-writing stories and 615 argumentative-writing essays.

\textbf{Baselines in Coauthor dataset}. Since pretrained checkpoints for SegFormer are not publicly available, we train their model from scratch. We split the CoAuthor dataset into training (60\%), validation (20\%), and test (20\%) sets, and compute all metrics on the held-out test set. 

When implementing our methods, we apply a clustering step with $k=3$, after the change points are estimated. The cluster with the highest mean is assigned to the fully LLM-generated class, the cluster with the lowest mean is assigned to the human-written class, and the remaining cluster is assigned to the collaboratively written class.

\textbf{Hardware setting}. All experiments are conducted on an H20 96 GB GPU with 96 Intel(R) Xeon(R) Platinum 8255C CPUs @ 2.50 GHz. 

\end{document}